\documentclass{article}

\usepackage{microtype}
\usepackage{graphicx}
\usepackage{subfigure}
\usepackage{booktabs} 

\usepackage[accepted]{icml2020}

\icmltitlerunning{Optimal approximation for unconstrained non-submodular minimization}
  
\PassOptionsToPackage{hyphens}{url}
\usepackage{color}
 \usepackage{tikz} 
 \usetikzlibrary{arrows,automata}
\usepackage{times,amsmath,amssymb,amsfonts,epsfig,graphicx, mathrsfs}
\usepackage{amsthm}  
\usepackage{thmtools,thm-restate}
 \usepackage{dsfont}
 \usepackage{enumitem}
\usepackage{footnote}
\usepackage{multicol}
\usepackage{natbib}
\usepackage{hyperref}

\newtheorem{definition}{Definition}

\newtheorem{proposition}{Proposition}
\newtheorem{lemma}{Lemma}
\newtheorem{corollary}{Corollary}

\newtheorem{remark}{Remark}

\usepackage[T3,T1]{fontenc}
\DeclareSymbolFont{tipa}{T3}{cmr}{m}{n}
\DeclareMathAccent{\invbreve}{\mathalpha}{tipa}{16}
\makesavenoteenv{table}
\makesavenoteenv{tabular}

\graphicspath{{figures/}}

\newcommand{\conv}{\mathrm{conv}}

\newcommand{\supp}{\mathrm{supp}}

\newcommand{\tF}{\tilde{F}}
\newcommand{\tG}{\tilde{G}}

\newcommand{\R}{\mathbb{R}}
\newcommand{\Kb}{\boldsymbol{K}}
\newcommand{\kb}{\boldsymbol{k}}

\newcommand{\D}{\mathcal{D}}

\newcommand{\X}{\mathcal{X}}

\newcommand{\C}{\mathcal{C}}

\newcommand{\N}{\mathcal{N}}

\newcommand{\A}{\boldsymbol{A}}

\newcommand{\Eb}{\mathbb{E}}

\newcommand{\tH}{\tilde{H}}
\newcommand{\I}{\boldsymbol{I}}

\newcommand{\x}{\boldsymbol x}
\newcommand{\uu}{\boldsymbol u}
\newcommand{\y}{\boldsymbol y}

\newcommand{\s}{\boldsymbol s}
\newcommand{\z}{\boldsymbol z}
\newcommand{\w}{\boldsymbol w}
\newcommand{\ab}{\boldsymbol a}

\newcommand{\kappabf}{\boldsymbol{\kappa}}
\newcommand{\tkappa}{\tilde{\kappa}}

\newcommand{\1}{\mathds{1}}
\newcommand{\0}{\boldsymbol{0}}

\DeclareMathOperator*{\argmin}{arg\,min}

\newif \ifprivate 
\privatetrue

\begin{document}

\twocolumn[
\icmltitle{Optimal approximation for unconstrained \\non-submodular minimization}



\icmlsetsymbol{equal}{*}

\begin{icmlauthorlist}
\icmlauthor{Marwa El Halabi}{MIT}
\icmlauthor{Stefanie Jegelka}{MIT}
\end{icmlauthorlist}

\icmlaffiliation{MIT}{Massachusetts Institute of Technology}

\icmlcorrespondingauthor{Marwa El Halabi}{marwash@mit.edu}

\icmlkeywords{Machine Learning, non-submodular minimization, approximate submodularity, Lov\'asz extension, convex closure, structured sparse learning, feature selection, variance reduction}

\vskip 0.3in
]



\printAffiliationsAndNotice{}  

\begin{abstract}
  Submodular function minimization is well studied, and existing algorithms solve it exactly or up to arbitrary accuracy. However, in many applications, such as structured sparse learning or batch Bayesian optimization, the objective function is not exactly submodular, but close. In this case, no theoretical guarantees exist. Indeed, submodular minimization algorithms rely on intricate connections between submodularity and convexity. We show how these relations can be extended to obtain approximation guarantees for minimizing non-submodular functions, characterized by how close the function is to submodular. We also extend this result to noisy function evaluations. Our approximation results are the first for minimizing non-submodular functions, and are optimal, as established by our matching lower bound.
\end{abstract}

\section{Introduction}

Many machine learning problems can be formulated as minimizing a \emph{set function} $H$. This problem is in general NP-hard, and can only be solved efficiently with additional structure. One especially popular example of such structure is that $H$ is \emph{submodular}, i.e., it satisfies the {diminishing returns (DR)} property: $H(A \cup \{i\})- H(A) \geq H(B \cup \{i\}) - H(B)$, for all $A \subseteq B, i \in V\setminus B$. Several existing algorithms minimize a submodular $H$ in polynomial time, exactly or within arbitrary accuracy.  
 Submodularity is a natural model for a variety of applications, such as image segmentation \cite{Boykov2004}, data selection \cite{Lin2010}, or clustering \cite{Narasimhan2005}. 
But, in many other settings, 
 such as structured sparse learning, Bayesian optimization, and column subset selection, the objective function is not exactly submodular. Instead, it satisfies a weaker version of the diminishing returns property.  An important class of such functions are \emph{$\alpha$-weakly DR-submodular} functions, introduced in \cite{Lehmann2006}. The parameter $\alpha$ quantifies how close the function is to being submodular (see Section~\ref{sect:preliminaries} for a precise definition). 
Furthermore, in many cases, only noisy evaluations of the objective are available. 
Hence, we ask:
  \emph{Do submodular minimization algorithms extend to such non-submodular noisy functions?}

Non-submodular \emph{maximization}, under various notions of approximate submodularity, has recently received a lot of attention \cite{Das2011, Elenberg2018, Sakaue2019, Bian2017a, Chen2017, Gatmiry2019, Harshaw2019, Kuhnle2018, Horel2016, Hassidim2018}. 
In contrast, 
only few studies consider \emph{minimization} of non-submodular set functions. 
%
%
Recent works have studied the problem of minimizing the ratio of two set functions, where one \cite{Bai2016,Qian2017} or both \cite{Wang2019} are non-submodular.
The ratio problem is related to constrained minimization, which does not admit a constant factor approximation even in the submodular case \cite{Svitkina2011}.
If the objective is approximately \emph{modular}, i.e., it has bounded \emph{curvature}, 
algorithmic techniques related to those for submodular maximization
achieve optimal approximations for constrained minimization \cite{Sviridenko2017,Iyer2013b}. Algorithms for minimizing the difference of two submodular functions were proposed in \cite{Iyer2012, kawahara2015}, but no approximation guarantees were provided.

In this paper, we study the unconstrained non-submodular minimization problem
\begin{equation}\label{eq:NonSubMin}
\min\nolimits_{S \subseteq V}\; H(S):= F(S) - G(S),
\end{equation}
where $F$ and $G$ are monotone (i.e., non-decreasing or non-increasing) functions, $F$ is \emph{$\alpha$-weakly DR-submodular}, and $G$ is \emph{$\beta$-weakly DR-supermodular}, i.e.,  $-G$ is ${\beta^{-1}}$-weakly DR-submodular. The definitions of weak DR-sub-/supermodularity only hold for monotone functions, and thus do not directly  apply to $H$.
We show that, perhaps surprisingly, {any} set function $H$ can be decomposed into functions $F$ and $G$ that satisfy these assumptions, albeit with properties leading to weaker approximations when the function is far from being submodular.

A key strategy for minimizing submodular functions exploits a tractable tight convex relaxation that enables the use of convex optimization algorithms. But, this relies on the equivalence between the convex closure of a submodular function and the polynomial-time computable \emph{Lov\'asz extension}. In general, the convex closure of a set function is NP-hard to compute, and the Lov\'asz extension is convex if and only if the set function is submodular.
Thus, the optimization delicately relies on submodularity; generally, a tractable tight convex relaxation is impossible. Yet, in this paper, we show that for approximately submodular functions, the Lov\'asz extension 
can be approximately minimized using a projected subgradient method (PGM). In fact, this strategy is guaranteed to obtain an approximate solution to Problem~\eqref{eq:NonSubMin}. 
This insight broadly expands the scope of submodular minimization techniques.
%
In short, our main contributions are:
\vspace{-7pt}
\begin{itemize}\setlength{\itemsep}{-1pt}
\item the \emph{first} approximation guarantee for unconstrained non-submodular  minimization characterized by closeness to submodularity: PGM achieves a {tight} approximation of
 $H(S) \leq {F(S^*)}/\alpha - \beta G(S^*) + \epsilon$; 
\item an extension of this result to the case where only a noisy oracle of $H$ is accessible; 
\item a hardness result showing that improving on this approximation guarantee 
would require exponentially many queries in the value oracle model;
\item applications to structured sparse learning and variance reduction in batch Bayesian optimization, implying the \emph{first} approximation guarantees for these problems; 
\item experiments demonstrating the robustness of classical submodular minimization algorithms  against noise and non-submodularity, reflecting our theoretical results. 
\end{itemize}

\section{Preliminaries} \label{sect:preliminaries}

We begin by introducing our notation, the definitions of weak DR-submodularity/supermodularity, and by reviewing some facts about classical submodular minimization.

\paragraph{Notation}
Let $V = \{1, \cdots, d\}$ be the ground set. 
Given a set function $F: 2^V \to \R$, we denote the \emph{marginal gain} of adding an element $i$ to a set $A$ by $F( i | A) = F( A  \cup \{i\}) - F(A)$.
Given a vector $\x \in \R^d$, $x_i$ is its $i$-th entry and $\supp(\x) = \{ i \in V | x_i \not = 0\}$ is its support set; $\x$ also defines a \emph{modular} set function as $\x(A) = \sum_{i \in A} x_i$. \vspace{-5pt}
\paragraph{Set function classes}
The function $F$ is \emph{normalized} if $F(\emptyset) = 0$, and
non-decreasing (non-increasing) if $F(A) \leq F(B)$ ($F(A) \geq F(B)$) for all $A \subseteq B$.
$F$ is \emph{submodular} if it has diminishing marginal gains: $F( i| A) \geq F(i | B)$ for all $A \subseteq B$, $i \in V\setminus B$, \emph{modular} if the inequality holds as an equality, and \emph{supermodular} if $F( i | A) \leq F(i | B)$. 
Relaxing these inequalities leads to the notions of weak DR-submodularity/supermodularity introduced in \cite{Lehmann2006} and \cite{Bian2017a}, respectively.

\begin{definition}[Weak DR-sub/supermodularity] \label{def:WDR}
A set function $F$ is \emph{$\alpha$-weakly DR-submodular}, with $\alpha > 0$, if 
\[F( i | A) \geq \alpha F( i | B), \text{ for all $ A \subseteq B, i \in V \setminus B$}. \]
Similarly, $F$ is $\beta$-weakly DR-supermodular, with $\beta > 0$, if 
\[ F( i | B)  \geq \beta F( i | A), \text{ for all $A \subseteq B, i \in V \setminus B$}. \]
We say that $F$ is \emph{$(\alpha,\beta)$-weakly DR-modular} if it satisfies both properties.
\end{definition}
If $F$ is non-decreasing, then $\alpha, \beta \in (0,1]$, and if it is non-increasing, then $\alpha, \beta \geq 1$. 
 $F$ is submodular (supermodular) iff $\alpha = 1$ ($\beta = 1$) and modular iff both $\alpha = \beta = 1$.\\
The parameters $1 \!-\! \alpha$ and $1 \!-\! \beta$ are referred to as \emph{generalized inverse curvature} \cite{Bogunovic2018} and \emph{generalized curvature} \cite{Bian2017a}, respectively. They extend the notions of \emph{inverse curvature} and \emph{curvature} \cite{Conforti1984} commonly  defined  for supermodular  and  submodular functions.
These notions are also related to \emph{weakly sub-/supermodular} functions \cite{Das2011,Bogunovic2018}.
Namely, the classes of weakly DR-sub-/super-/modular functions are respective subsets of the classes of  weakly sub-/super-/modular functions \citep[Prop. 8]{ElHalabi2018}, \citep[Prop. 1]{Bogunovic2018}, as illustrated in Figure~\ref{fig:SetClasses}. For a survey of other notions of approximate submodularity, we refer the reader to \citep[Sect. 6]{Bian2017a}.

\begin{figure}\label{fig:SetClasses}
\centering 
 \includegraphics[trim=0 190 0 180, clip, scale=.20]{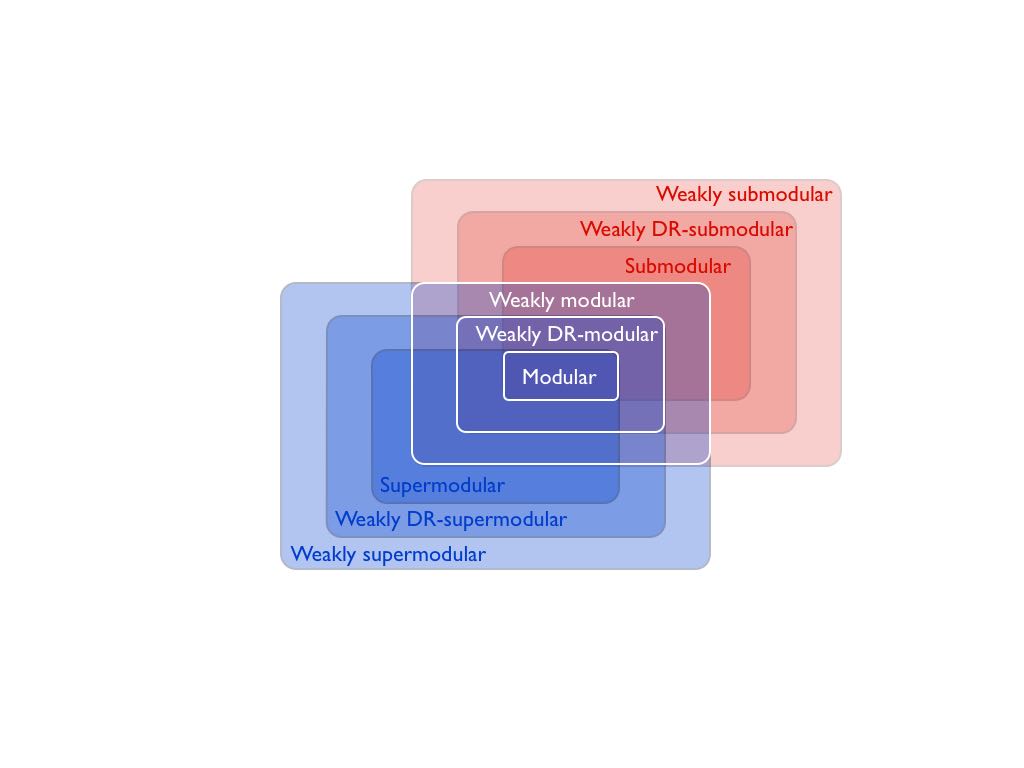}
 \vspace{-5pt}
 \caption{Classes of set functions}
\end{figure}

\paragraph{Submodular minimization}\label{sect:SFM}
Minimizing a submodular set function $F$ is equivalent to minimizing a non-smooth convex function that is given by a \emph{continuous extension} of $F$, i.e., a continuous interpolation of $F$ on the full hypercube $[0, 1]^d$. This extension, called the \emph{Lov\'asz extension} \cite{Lovasz1983},  is convex if and only if $F$ is submodular. 
%
%
\begin{definition}[Lov\'asz extension] \label{Def:LE}
Given any normalized set function $F$, its {Lov\'asz extension} $f_L: \R^d \to \R$ is defined as 
$$f_L(\s) = \sum_{k = 1}^d s_{j_k} F( j_k | S_{k-1}),$$
where $s_{j_1} \geq \cdots \geq s_{j_d}$ 
are the sorted entries of $\s$ in decreasing order, 
and $S_k =\{j_1, \cdots, j_k\}$.
\end{definition}

Minimizing $f_L$ is equivalent to minimizing $F$ . Moreover, when $F$ is submodular,
a subgradient $\kappabf$ of $f_L$  at any $\s \in \R^d$ can be computed efficiently by sorting the entries of $\s$ in decreasing order and taking $\kappa_{j_k} =  F( j_k | S_{k-1})$ for all  $k \in V$ \cite{Edmonds2003}.
This relation between submodularity and convexity allows for generic convex optimization algorithms to be used for minimizing $F$. However, it has been unclear how these relations are affected if the function is only approximately submodular. In this paper, we give an answer to this question.


\section{Approximately submodular minimization}\label{sect:Algo}

We consider set functions $H: 2^V \to \R$ of the form $H(S) = F(S) - G(S)$, where
$F$ is $\alpha$-weakly DR-submodular,  $G$ is $\beta$-weakly DR-supermodular, and both $F$ and $G$ are  normalized  non-decreasing  functions.
We later extend our results to non-increasing functions. 
We assume a \emph{value oracle} access to $H$; i.e., there is an oracle that, given a set $S \subseteq V$, returns the value $H(S)$.
Note that $H$ itself is in general \emph{not} weakly DR-submodular. 
Interestingly, any set function can be decomposed in this form.

\begin{restatable}{proposition}{primeDecomposition}\label{prop:Decomposition-alpha-gen}
Given a set function $H$, and  $\alpha, \beta \in (0,1]$ such that $\alpha \beta < 1$, there exists a non-decreasing $\alpha$-weakly DR-submodular function $F$, and a non-decreasing $(\alpha,\beta)$-weakly DR-modular function $G$, such that
 $H(S)  = F(S) - G(S)$ for all $S \subseteq V$.
\end{restatable}
\begin{proof}[Proof sketch]
This decomposition builds on the decomposition of $H$ into the difference of two non-decreasing submodular functions  \cite{Iyer2012}. 
We start by choosing any function $G'$ which is non-decreasing $(\alpha,\beta)$-weakly DR-modular, and is strictly $\alpha$-weakly DR-submodular, i.e., $\epsilon_{G'}= \min_{i \in V, A \subset B \subseteq V \setminus i} G'(i|A) - \alpha G'(i|B) >0$. 
It is not possible to choose $G'$ such that $ \alpha = \beta = 1 $  (this would imply $G'(i|B) \geq G'(i|A) > G'(i|B)$). We  then construct $F$ and $G$ based on $G'$.\\
Let  $\epsilon_H = \min_{i \in V, A \subseteq B \subseteq V \setminus i} H(i|A)  - \alpha H(i|B) < 0$ be the violation of $\alpha$-weak DR-submodularity of $H$; we may use a lower bound $\epsilon'_H \leq \epsilon_H$.
We define $F'(S) =  H(S) + \tfrac{|\epsilon'_H|}{\epsilon_{G'}} G'(S)$.
$F'$ is not necessarily non-decreasing. To correct for that, let  $V^- = \{ i : F'(i|V \setminus i) < 0\}$ 
and define $F(S) = F'(S)-  \sum_{i \in S \cap V-} F'(i|V \setminus i)$.
We can show that $F$ is non-decreasing $\alpha$-weakly DR-submodular.
We also define $G(S) =  \tfrac{|\epsilon'_H|}{\epsilon_{G'}} G'(S) -  \sum_{i \in S \cap V-} F'(i|V \setminus i)$, then $G$ is non-decreasing $(\alpha, \beta)$-weakly DR-modular, and  $H(S) = F(S) - G(S)$.
\end{proof}

Proposition \ref{prop:Decomposition-alpha-gen} generalizes the result of \citep[Theorem 18]{cunningham1983} showing that any submodular function can be decomposed into the difference of a non-decreasing submodular function and a non-decreasing modular function. When $H$ is submodular, the decomposition in Proposition 1 recovers the one from \cite{{cunningham1983}}, by simply choosing $\alpha = \beta = 1$. The resulting violation of submodularity is $\epsilon_H = 0$, and $G'$ is not needed.

Computing such a decomposition is \emph{not} required to run PGM for minimization; it is only needed to evaluate the corresponding approximation guarantee. 
The construction in the above proof uses the maximum violation $\epsilon_H$ of $\alpha$-weak DR-submodularity of $H$, which is NP-hard in general. However, when $\epsilon_H$ or a lower bound of it is known, $F$ and $G$ can be obtained in polynomial time, for a suitable choice of $G'$. Proposition \ref{ex:(1,beta)-mod} provides a valid choice of $G'$ for $\alpha = 1$. Any modular function can be used for $\alpha < 1$. 

\begin{restatable}{proposition}{primePropExWeakMod}\label{ex:(1,beta)-mod}
Given $\beta \in (0,1)$, let $G'(S) = g(|S|)$ where $g(x) = \tfrac{1}{2} a x^2 + (1- \tfrac{1}{2} a) x$ with $a = \tfrac{\beta - 1}{d-1}$. Then
 $G'$ is non-decreasing $(1,\beta)$-weakly DR-modular, and is strictly submodular, with $\epsilon_{G'}= \min_{i \in V, A \subset B \subseteq V \setminus i} G'(i|A) -  G'(i|B)=  -a >0$. 
\end{restatable}

The lower bound on $\epsilon_H$ and the choice of $\alpha, \beta$ and  $G'$ will affect the approximation guarantee on $H$, as we clarify later. When $H$ is far from being submodular, it may not be possible to choose $G'$ to obtain a non-trivial guarantee. 
However, many important non-submodular functions do admit a decomposition which leads to non-trivial bounds. 
We call such functions \emph{approximately} submodular, and provide some examples in Section \ref{sect:Application}.

In what follows, we establish a connection between approximate submodularity and approximate convexity, which allows us to derive a  \emph{tight} approximation guarantee for PGM 
on  Problem \eqref{eq:NonSubMin}. 
All omitted proofs are in the Supplement. 

\subsection{Convex relaxation}\label{sect:ConvRel}
When $H$ is not submodular, the connections between its Lov\'asz extension and tight convex relaxation for exact minimization, outlined in Section \ref{sect:SFM}, break down.
However, Problem \eqref{eq:NonSubMin} can still be converted to a non-smooth 
 convex optimization problem, via a different convex extension. Given a set function $H$, its \emph{convex closure} $h^-$ is the point-wise largest convex function from $[0,1]^d$ to $\R$ that always lower bounds $H$. Intuitively, $h^-$ is the \emph{tightest} convex extension of $H$ on $[0,1]^d$. The following equivalence holds
 \citep[Prop. 3.23]{Dughmi2009}:
\begin{equation}\label{eq:DisctCont}
\min_{S \subseteq V} H(S) = \min_{\s \in [0,1]^d} h^-(\s).
\end{equation}
Unfortunately, evaluating and optimizing $h^-$ for a general set function is NP-hard \cite{Vondrak2007}.
The key property that makes Problem \eqref{eq:DisctCont} efficient to solve when $H$ is  submodular is that its convex closure then coincides with its tractable {Lov\'asz
extension}, i.e., $h^- = h_L$.
This equivalence no longer holds if $H$ is only approximately submodular. 
But, in this case, a weaker key property holds: Lemma \ref{lem:ModularApprox} shows that the {Lov\'asz extension} approximates the convex closure $h^-$, and that the same vectors that served as its subgradients in the submodular case
can serve as  approximate subgradients to $h^-$. 

\begin{restatable}{lemma}{primeModApproxLemma}\label{lem:ModularApprox}
Given a vector $\s \in [0,1]^d$ such that $s_{j_1} \geq  \cdots \geq  s_{j_d}$,
we define $\kappabf$ such that $\kappa_{j_k} =  H( j_k | S_{k-1})$ where $S_k = \{j_1, \cdots, j_k\}$. Then, $h_L(\s) = \kappabf^\top \s \geq h^-(\s)$, and
\begin{align*}
\kappabf(A) &\leq \tfrac{1}{\alpha} F(A) - \beta G(A) \text{ for all } A \subseteq V,  \\
\kappabf^\top \s' &\leq \tfrac{1}{\alpha}  f^-(\s')+ \beta (-g)^-(\s') \text{ for all } \s' \in [0,1]^d.
\end{align*}

\end{restatable}
To prove Lemma~\ref{lem:ModularApprox}, we use a specific formulation of the convex closure $h^-$  \citep[Def. 20]{ElHalabi2018a}:
\begin{equation*}\label{eq:ClosureMaxForm}
h^-(\s) = \max_{\kappabf \in \R^d, \rho \in \R} \{  \kappabf^\top \s + \rho : \kappabf(A) + \rho \leq H(A), \forall A \subseteq V\},
\end{equation*}
and build on the proof of Edmonds' greedy algorithm \cite{Edmonds2003}.
We can view the vector $\kappabf$ in Lemma  \ref{lem:ModularApprox} as an approximate subgradient of $h^-$ at $\s$ in the following sense:
$$\tfrac{1}{\alpha} f^-(\s')+ \beta (-g)^-(\s') \geq h^-(\s) + \langle \kappabf, \s' - \s \rangle, \forall \s' \in [0,1]^d.$$
Lemma  \ref{lem:ModularApprox} also implies that the Lov\'asz extension $h_L$ approximates the convex closure $h^-$ in the following sense: 
$$h^-(\s) \leq h_L(\s) \leq \tfrac{1}{\alpha} f^-(\s) + \beta (-g)^-(\s), \forall \s \in  [0,1]^d.$$
We can thus say that $h_L$ is approximately convex in this case.
This key insight allows  us to approximately minimize $h^-$ via convex optimization algorithms. 

\subsection{Algorithm and approximation guarantees}\label{sect:AlgoGaurantee}
Equipped with the approximate subgradients of $h^-$, we can now apply an approximate projected subgradient method (PGM). Starting from an arbitrary $\s^1 \in [0,1]^d$, PGM iteratively updates $\s^{t+1} = \Pi_{[0,1]^d}(\s^t - \eta \kappabf^t)$, where $\kappabf^t$ is the approximate subgradient at $\s^t$ from Lemma~\ref{lem:ModularApprox}, and $\Pi_{[0,1]^d}$ is the projection onto $[0,1]^d$. 
We set the step size to $\eta = \tfrac{R }{L \sqrt{T}}$, where $L =  F(V) +G(V)$ is the Lipschitz constant, i.e.,  $\|  \kappabf^t\|_2 \leq L$ for all $t$, and $R= 2 \sqrt{d}$ is the domain radius $\| \s^1 - \s^* \|_2 \leq R$. 

\begin{restatable}{theorem}{primeNonSubTheom}\label{them:NonSubMin2}
After $T$ iterations of PGM, $\hat{\s} \in \argmin_{t \in \{1, \cdots, T\}}   h_L(\s^t)$ satisfies:
$$  h^-(\hat{\s}) \; \leq\;   h_L(\hat{\s}) \; \leq\; \frac{1}{\alpha} f^-(\s^*) + \beta (-g)^-(\s^*)  + \frac{R L}{\sqrt{T}},$$
where $\s^*$ is an optimal solution of $\min_{\s \in [0,1]^d} h^-(\s)$.
\end{restatable}


Importantly, the algorithm does not need to know the parameters $\alpha$ and $\beta$, which can be hard to compute in practice. In fact, its iterates are exactly the same as in the submodular case.
Theorem \ref{them:NonSubMin2} provides an approximate fractional solution $\hat{\s} \in [0,1]^d$. To round it to a discrete solution, Corollary \ref{corr:thresholding} shows that it is sufficient to pick the superlevel set of $\hat{\s}$ with the smallest $H$ value.
\begin{restatable}{corollary}{primeCorrThresholding}\label{corr:thresholding}
Given the fractional solution $\hat{\s}$ in Theorem \ref{them:NonSubMin2}, let $\hat{S}_k = \{j_1, \cdots, j_k\}$ such that $\hat{s}_{j_1} \geq \cdots \geq  \hat{s}_{j_d}$, and $\hat{S}_0 = \emptyset$. Then $\hat{S} \in \argmin_{k \in \{0,\cdots, d\}}  H(\hat{S}_k)$ satisfies 
$$ H(\hat{S}) \leq \frac{1}{\alpha} F(S^*) - \beta G(S^*)  + \frac{R L}{\sqrt{T}},$$
where $S^*$ is an optimal solution of Problem \eqref{eq:NonSubMin}.
\end{restatable}

To obtain a set that satisfies $H(\hat{S}) \leq F(S^*)/\alpha  - \beta G(S^*) + \epsilon$, we thus need at most $O({d L^2}/{\epsilon^2})$ iterations of PGM, where the time per iteration is $O(d \log d + d ~\text{EO})$,  with EO being the time needed to evaluate $H$ on any set. Moreover, the techniques from \cite{Chakrabarty2017, Axelrod2019} for accelerating the runtime of stochastic PGM to $\tilde{O}(d ~\text{EO}/\epsilon^2)$ can be extended to our setting.

If $F$ is regarded as a cost and $G$ as a revenue, this guarantee states that the returned solution achieves at least a fraction $\beta$ of the revenue of the optimal solution, by paying at most a $1/\alpha$-multiple of the cost.
The quality of this guarantee depends on $F, G$ and their parameters $\alpha, \beta$; it
becomes vacuous when $F(S^*)/\alpha  \geq \beta G(S^*)$.
If $H$ is submodular, Problem \eqref{eq:NonSubMin} reduces to submodular minimization and Corollary \ref{corr:thresholding} recovers the guarantee $H(\hat{S}) \leq {H(S^*)} + {R L}/{\sqrt{T}}$. 

\begin{remark}\label{rmk:parameters}
The upper bound in Corollary \ref{corr:thresholding} still holds
 if  the worst case parameters $\alpha, \beta$ are instead replaced by $\alpha_T = \tfrac{1}{T} \sum_{t = 1}^T \tfrac{F(S^*)}{ \kappabf_F^{t}(S^*)}$ and $\beta_T = \tfrac{1}{T} \sum_{t = 1}^T \tfrac{ \kappabf_G^t(S^*)}{G(S^*)}$, where $(\kappabf^t_F)_{j_k} = F(j^t_k | S^t_{k-1})$ and  $(\kappabf^t_G)_{j_k} = G(j^t_k | S^t_{k-1})$. This refined upper bound yields improvements if only few of the relevant submodularity inequalities are violated. 
\end{remark}

All results in this section extend to the case where $F$ and $G$ are non-increasing functions. 

\begin{restatable}{corollary}{primeNonincreasing}\label{corr:nonincreasing}
Given $H(S) = F(S) - G(S)$, where $F$ and $G$ are non-increasing functions with $F(V) = G(V) = 0$, we run PGM with $\tH(S) = H(V \setminus S)$ for $T$ iterations. Let  $\tilde{\s} \in \argmin_{t \in \{1, \cdots, T\}}   \tilde{h}_L(\s^t)$ and  $\hat{S} = V \setminus \tilde{S}$, where $\tilde{S}$ is the superlevel set of $\tilde{\s}$ with the smallest $H$ value, then
$$H(\hat{S} ) \leq \alpha {F(S^*)}  - \frac{1}{\beta} G(S^*) + \frac{R L}{\sqrt{T}},$$
where $S^*$ is an optimal solution of Problem \eqref{eq:NonSubMin}.
\end{restatable}

For a general set function $H$, using $F$ and $G$ from the decomposition in Proposition~\ref{prop:Decomposition-alpha-gen}, yields in Corollary \ref{corr:thresholding}:
{\small
$$ H(\hat{S}) \leq \tfrac{1}{\alpha} H(S^*) + (\tfrac{1}{\alpha} -\beta) \bigl(  \tfrac{|\epsilon'_H|}{\epsilon_{G'}}  G'(S^*) - \!\!\!\!\! \sum_{i \in S \cap V^-} \!\!\!\!\!  F'(i | V \setminus i) \bigl) + \epsilon,$$}
where $\epsilon'_H$ is a lower bound on the violation of $\alpha$-weak DR-submodularity of $H$, $F'$ and $G'$ are the auxilliary functions used to construct $F$ and $G$, and $\epsilon_{G'}$ is the strict $\alpha$-weak DR-submodularity of $G'$ (see proof of Proposition \ref{prop:Decomposition-alpha-gen} for precise definitions). 
It is clear that a larger lower bound $|\epsilon'_H|$
 worsens the upper bound on $H(\hat{S})$. Moreover, the choice of $G'$ affects the bound: ideally, we want to choose $G'$ to minimize $G'(S^*)$, and maximize
the quantities $\alpha, \epsilon_{G'}$ and $\beta$, which characterize how submodular and supermodular $G'$ is, respectively. However,
a larger $\alpha$ leads to a larger $|\epsilon'_H|$ and smaller $\epsilon_{G'}$, and a larger $\epsilon_{G'}$ would result in a smaller $\beta$, and vice versa. The best choice of $G'$ will depend on $H$.

In Appendix \ref{sect:AppLowerBd}, we provide an example showing that the approximation guarantees in Corollary \ref{corr:thresholding} and \ref{corr:nonincreasing}  are \emph{tight}, i.e., they cannot be improved for PGM, even if $F$ and $G$ are weakly DR-modular. Furthermore, in Section \ref{sec:lowerbd} we show that these approximation guarantees are \emph{optimal} in general.
Apart from the above results for general unconstrained minimization, our results also imply approximation guarantees for generalizing constrained submodular minimization to weakly DR-submodular functions. We discuss this extension in Appendix \ref{sect:ConstMin}.

\subsection{Extension to noisy evaluations} \label{sect:NoisySFM}

In many real-world applications, we do not have access to the objective function itself, but rather to a noisy version of it.
Several works have considered maximizing noisy oracles of submodular  \cite{Horel2016, Singla2016, Hassidim2017, Hassidim2018} and weakly submodular  \cite{Qian2017a} functions. In contrast, to the best of our knowledge, \emph{minimizing} noisy oracles of submodular functions was only studied in \cite{Blais2018}.  

We address a more general setup where the underlying function $H$ is not necessarily submodular. We assume again that $F$ and $G$ are normalized and non-decreasing. The results easily extend to non-increasing functions as in Corollary~\ref{corr:nonincreasing}. 
We show in Proposition \ref{prop:approximateOracle} that our approximation guarantee for Problem \eqref{eq:NonSubMin} continues to hold when we only have access to an approximate oracle $\tH$. 
 Essentially, $\tH$ still allows to obtain approximate subgradients of $h^-$ in the sense of Lemma \ref{lem:ModularApprox}, but now with an additional additive error.

\begin{restatable}{proposition}{primeapproximateOracle} \label{prop:approximateOracle}
Assume we have an approximate oracle $\tilde{H}$ with input parameters $\epsilon, \delta \in (0,1)$, such that for every $S \subseteq V$, $|\tilde{H}(S) - H(S)| \leq \epsilon$ with probability $1 - \delta$. We run PGM with $\tH$ for $T$ iterations.
Let $\hat{\s}  = \argmin_{t \in \{1,\cdots, T\}} \tilde{h}_L(\s^t)$, 
and $\hat{S}_k = \{j_1, \cdots, j_k\}$ such that $\hat{s}_{j_1} \geq \cdots \geq  \hat{s}_{j_d}$. 
 Then $\hat{S} \in  \argmin_{k \in \{0,\cdots, d\}}  \tH(\hat{S}_k)$ satisfies
$$ H(\hat{S}) \leq  \tfrac{1}{\alpha} F(S^*) - \beta G(S^*) + \epsilon',$$
with probability $1 - \delta'$, by choosing $\epsilon = \tfrac{\epsilon'}{8 d }$,  $\delta = \tfrac{\delta' \epsilon'^2}{32 d^2}$ and using $2 T d$ calls to $\tilde{H}$ with $T = ({4 \sqrt{d} L}/{\epsilon'})^2$.
\end{restatable}
%
Blais et al \yrcite{Blais2018} consider the same setup for the special case of submodular $H$, 
and use the cutting plane method of \cite{Lee2015}. Their runtime has better dependence $O(\log(1/\epsilon'))$ on the error $\epsilon'$, but worse dependence $O(d^3)$ on the dimension $d = |V|$, and their result needs oracle accuracy $\epsilon =  O(\epsilon'^2/ d^5)$.
Hence, for large ground set sizes $d$, Proposition~\ref{prop:approximateOracle} is preferable.
This proposition allows us, in particular, to handle multiplicative and additive noise in $H$. 
%
\begin{restatable}{proposition}{primeInconsistentNoiseSub}
\label{prop:InconsistentNoiseSub}
Let $\tH = \xi H$ where the noise $\xi \geq 0$ is bounded by $|\xi| \leq \omega$ and is independently drawn from a distribution $\D$ with mean $\mu > 0$. We define the function $\tH_m$ as the mean of $m$ queries to $\tH(S)$. $\tH_m$ is then an approximate oracle to $\mu H$. 
In particular, for every $\delta, \epsilon \in (0,1)$, taking $m =  ({ \omega H_{\max}}/{\epsilon})^2 \ln({1}/{\delta}) $ where 
 $H_{\max} = \max_{S \subseteq V} H(S)$, we have for every $S \subseteq V$, $| \tH_m( S) - \mu H(S) | \leq \epsilon$ with probability at least $1-\delta$.
\end{restatable}

Propositions~\ref{prop:approximateOracle} and \ref{prop:InconsistentNoiseSub} imply that by using PGM with $\tH_m$ and picking the superlevel set with the smallest $\tH_m$ value,
 we can find a set $\hat{S}$ such that $ H(\hat{S}) \leq  F(S^*)/\alpha  - \beta G(S^*) + \epsilon'$ with probability $1 - \delta'$, using $m = O\bigl( (\tfrac{  \omega H_{\max} d}{\mu \epsilon'})^2 \ln(\tfrac{ d^2}{\delta' \mu^2 \epsilon'^2}) \bigl)$ samples, after $T = O \bigl( (\sqrt{d} \mu H_{\max}/\epsilon')^2 \bigl)$ iterations,  with $O\bigl( \tfrac{\omega}{\mu} (\tfrac{   H_{\max} d}{ \epsilon'})^4 \ln(\tfrac{ d^2}{\delta' \mu^2 \epsilon'^2}) \bigl)$ total calls to $\tH$. 
 Note that $H_{\max}$ is upper bounded by $F(V)$. 
This result provides a theoretical upper bound on the number of samples needed to be robust to bounded multiplicative noise. 
Much fewer samples  are actually needed in practice, as illustrated in our experiments (Section \ref{sect:ExpNoisy}).
 Using similar arguments, our results also extend to additive noise oracles $\tH = H + \xi$.
 
 \subsection{Inapproximability Result}\label{sec:lowerbd}

By Proposition \ref{prop:Decomposition-alpha-gen}, Problem \eqref{eq:NonSubMin} is equivalent to general set function minimization. Thus, solving it optimally or within any multiplicative approximation factor, i.e., $H(\hat{S}) \leq \gamma(d) H(S^*)$ for some positive polynomial time computable function $\gamma(d)$ of $d$, is NP-Hard \cite{Trevisan2004, Iyer2012}. 
Moreover, in the value oracle model, it is impossible to obtain any multiplicative constant factor approximation within a subexponential number of  queries \cite{Iyer2012}.
Hence, it is necessary to consider bicriteria-like approximation guarantees as we do. 

We now show that our approximation results are optimal: in the value oracle model, no algorithm with a subexponential number of  queries can improve on the approximation guarantees achieved by PGM, even when $G$ is weakly DR-modular.

\begin{restatable}{theorem}{primeThemLowerBd}\label{them:lowerbd}
For any $\alpha, \beta \in (0,1]$ such that $\alpha \beta <1, d > 2$ and $\delta >0$, there are instances of Problem~\eqref{eq:NonSubMin} such that no (deterministic or randomized) algorithm, using less than exponentially many queries, can always find a solution $S \subseteq V$ of expected value at most $F(S^*)/\alpha - \beta G(S^*) - \delta$. 
\end{restatable}
\begin{proof}[Proof sketch]
Our proof technique is similar to \cite{Feige2011}:
We randomly partition the ground set into $V = C \cup D$, and construct a normalized set function $H$ whose values depend only on $k(S) = |S \cap C|$ and $\ell(S)= |S \cap D|$: 
\begin{align*}
H(S) = \begin{cases}
0 &\text{if $|k(S) - \ell(S)| \leq \epsilon d$}\\
\tfrac{2 \alpha \delta}{ 2 - d }&\text{otherwise,}
\end{cases}
\end{align*}
for some $\epsilon \in [1/d, 1/2)$.
We use Proposition \ref{prop:Decomposition-alpha-gen} to decompose $H$ into the difference of a non-decreasing $\alpha$-weakly DR-submodular function $F$, and a non-decreasing $(\alpha,\beta)$-weakly DR-modular function $G$. 
We argue that, with probability $1 - 2\exp(-\frac{\epsilon^2 d}{4})$, any given query $S$ will be ``balanced'', i.e., $|k(S) - \ell(S)| \leq \epsilon d$. Hence no algorithm can distinguish between $H$ and the constant zero function, with subexponentially many  queries.  On the other hand, we have $H(S^*) = \frac{2 \alpha \delta}{ 2 - d } <0$, achieved at $S^* = C$ or $D$, and $\tfrac{1}{\alpha}F(S^*) - \beta G(S^*) - \delta <0$. Therefore, the algorithm cannot find a set with value $H(S) \leq F(S^*)/\alpha - \beta G(S^*) - \delta$.
\end{proof}


The approximation guarantees in Corollary \ref{corr:thresholding} and \ref{corr:nonincreasing} are thus optimal. 
In the above proof, $G$ belongs to the smaller class of weakly DR-modular functions, but $F$ not necessarily. 
Whether the approximation guarantee can be improved when $F$ is also weakly DR-modular is left as an open question. Yet,
the tightness result in Appendix \ref{sect:AppLowerBd} implies that 
such improvement cannot be achieved by PGM.


\section{Applications}\label{sect:Application}
Several applications can benefit from the theory in this work. We discuss two examples here, where we show that the objective functions have the form of Problem~\ref{eq:NonSubMin}, implying the \emph{first} approximation guarantees for these problems. Other examples include column subset selection  \cite{Sviridenko2017} and Bayesian A-optimal experimental design \cite{Bian2017a}, where $F$ is the cardinality function, and $G$ is weakly DR-supermodular with $\beta$ depending on the inverse of the condition number of the data matrix. 

\subsection{Structured sparse learning} \label{sect:StructSparse}

Structured sparse learning aims to estimate a \emph{sparse} parameter vector whose support satisfies a particular \emph{structure}, such as group-sparsity, clustering, tree-structure, or diversity \cite{Obozinski2016, Kyrillidis2015}. 
Such problems can be formulated as 
\begin{equation}\label{eq:structuredSparsity}
\min_{\x \in \R^d} \ell(\x) + \lambda F(\supp(\x)), \vspace{-5pt}
\end{equation}
where $\ell$ is a convex loss function and $F$ is a set function favoring the desirable supports. 
Existing convex methods propose to replace the discrete regularizer $F(\supp(\x))$ by its ``closest'' convex relaxation \cite{Bach2010, ElHalabi2015, Obozinski2016, ElHalabi2018}. For example, the cardinality regularizer $|\supp(\x)|$ is replaced by the $\ell_1$-norm.
This allows the use of standard convex optimization methods, but does not provide any approximation guarantee for the original objective function without statistical modeling assumptions. This approach is computationally feasible  only when $F$ is submodular \cite{Bach2010} or can be expressed as an integral linear program \cite{ElHalabi2015}.

 Alternatively, one may write Problem \eqref{eq:structuredSparsity} as  
\begin{equation}\label{eq:structuredSparsityDiscrete}
\min_{S \subseteq V} H(S) = \lambda F(S) - G^\ell(S),  \vspace{-5pt}
\end{equation}  where 
$G^\ell(S) = \ell(0) -  \min_{\supp(\x) \subseteq S} \ell(\x)$ is a normalized non-decreasing set function.
Recently, it was shown that
if $\ell$ has restricted smoothness and strong convexity, $G^\ell$ is weakly modular \cite{Elenberg2018,Bogunovic2018, Sakaue2018}.
This allows for approximation guarantees of greedy algorithms to be applied
to the constrained variant of Problem \eqref{eq:structuredSparsity}, but only 
for the special cases of a sparsity constraint \cite{Das2011, Elenberg2018} or some near-modular constraints \cite{Sakaue2019}. 

In applications, however, the structure of interest is often better modeled by a non-modular regularizer $F$, which may be submodular  \cite{Bach2010} or non-submodular \cite{ElHalabi2015, ElHalabi2018}. Weak modularity of $G^\ell$ is not enough to directly apply the result in Corollary \ref{corr:thresholding}, but, if the loss function $\ell$ 
is smooth, strongly convex, and is generated from random data, then we show that $G^\ell$ is also weakly DR-modular. 
\begin{restatable}{proposition}{primeLSWDRmodProp}\label{prop:LS-WDRmod}
Let  
$\ell(\x) \!\!= L(\x) - \z^\top \x$, where $L$ is smooth and strongly convex,
and $\z \in \R^d$ has a continuous density w.r.t  the Lebesgue measure. Then there exist $\alpha_G, \beta_G \!\! > \!\! 0$ such that $G^\ell$ is $(\alpha_G,\beta_G)$-weakly DR-modular, almost surely.
\end{restatable}
  
We prove Proposition \ref{prop:LS-WDRmod} by first utilizing a result from 
 \cite{Elenberg2018}, which relates the marginal gain of $G^\ell$ to the marginal decrease of $\ell$. 
We then argue that the minimizer of $\ell$, restricted to any given support, has full support  with probability one, and thus $\ell$ has non-zero marginal decrease with probability one. The proof is given in Appendix \ref{sect:AppApplicationProofs}. 
The actual $\alpha_G, \beta_G$ parameters depend on the conditioning of $\ell$. Their positivity also relies on $\z$ being random, typically, data drawn from a distribution 
\citep[Sect. A.1]{Sakaue2018}. 
%
%
In Section \ref{sect:ExpStructure}, we evaluate Proposition \ref{prop:LS-WDRmod} empirically. 

The approximation guarantee in Corollary \ref{corr:thresholding} thus applies directly to Problem \eqref{eq:structuredSparsityDiscrete}, whenever $\ell$ has the form in Proposition \ref{prop:LS-WDRmod}, and $F$ is $\alpha$-weakly DR-submodular. 
For example, this holds when $\ell$ is the least squares loss with a nonsingular measurement matrix. 
Examples of structure-inducing regularizers $F$
include submodular regularizers \cite{Bach2010}, and non-submodular ones such as the range cost function \cite{Bach2010, ElHalabi2018} ($\alpha = \tfrac{1}{d-1}$), which favors interval supports, with applications in time-series and cancer diagnosis  \cite{Rapaport2008}, and the cost function considered \cite{Sakaue2019} ($\alpha = \tfrac{1 + a}{1 + (b-a)}$, where $0 <  2 a <  b$ are cost parameters), which favors the selection of sparse and cheap features, with applications in healthcare.

\subsection{Batch Bayesian optimization}\label{sec:BO}

The goal in batch Bayesian optimization is to optimize an unknown expensive-to-evaluate noisy function $f$ with as few batches of  function evaluations as possible \cite{Desautels2014, Gonzalez2016}. For example, evaluations can correspond to performing expensive experiments. 
The evaluation points are chosen to maximize an acquisition function subject to a cardinality constraint. Several acquisition functions have been proposed for this purpose, amongst others  the \emph{variance reduction} function \cite{Krause2008, Bogunovic2016}. 
This function is used to maximally reduce the variance of the posterior distribution over potential maximizers of the unknown function. 

Often, the unknown $f$ is modeled by a Gaussian process with zero mean and kernel function $k(\x,\x')$, and we observe noisy evaluations $y = f(\x) + {\epsilon}$ of the function, where ${\epsilon} \sim \N(0,\sigma^2)$. Given a set $\X = \{\x_1, \cdots, \x_d\}$ of potential maximizers of $f$, each $\x_i \in \R^n$, and a set $S \subseteq V$,  let $\y_S=  [y_i]_{i\in S}$ be the corresponding observations at points $x_i, i \in S$. The posterior distribution of $f$ given $\y_S$ is again a Gaussian process, with  
variance $\sigma_S^2(\x) = k(\x,\x) -  \kb_S(\x)^\top (\Kb_S + \sigma^2 \I)^{-1} \kb_S(\x)$
where $\kb_S = [k(\x_i, \x)]_{i \in S}$, and $\Kb_S = [k(\x_i, \x_j)]_{i, j \in S}$ is the corresponding submatrix of the positive definite kernel matrix $\Kb$. 
The variance reduction function is defined as: 
$$G(S) = \sum_{i \in V} \sigma^2(\x_i)  - \sigma^2_{S}(\x_i),  \vspace{-5pt}$$ 
where $\sigma^2(\x_i)  = k(\x_i,\x_i)$. We show that the variance reduction function is  weakly DR-modular. 
\begin{proposition}\label{prop:VarRed-WDRMod}
 The variance reduction function $G$  is non-decreasing $(\beta, \beta)$-weakly DR-modular, with $\beta = (\frac{\lambda_{\min}(\Kb)}{\sigma^2 + \lambda_{\min}(\Kb) })^2\frac{\lambda_{\min}(\Kb)}{\lambda_{\max}(\Kb)}$, where $\lambda_{\max}(\Kb)$ and $\lambda_{\min}(\Kb)$ are the largest and smallest eigenvalues of $\Kb$.
\end{proposition}
To prove Proposition \ref{prop:VarRed-WDRMod}, we show that $G$ can be written as a noisy column subset selection objective, and prove that such an objective function is weakly DR-modular, generalizing the result of \cite{Sviridenko2017}. The proof is given in Appendix \ref{sect:AppBOApplicationProofs}. 
The variance reduction function can thus be maximized with a greedy algorithm to a $\beta$-approximation \cite{Sviridenko2017}, which follows from a stronger notion of approximate modularity.


Maximizing the variance reduction may also be phrased as an instance of Problem \eqref{eq:NonSubMin}, with $G$ being the variance reduction function, and $F(S) = \lambda |S|$ an item-wise cost. This formulation easily allows to include nonlinear costs with (weak) decrease in marginal costs (economies of scale). For example, in the sensor placement application, the cost of placing a sensor in a hazardous environment may diminish if other sensors are also placed in similar environments. Unlike previous works, the approximation guarantee in Corollary \ref{corr:thresholding} still applies to such cost functions, 
while maintaining the $\beta$-approximation with respect to $G$.


\section{Experiments} \label{sect:Experiments}
We empirically validate our results on  noisy submodular minimization  and structured sparse learning. In particular, we address the following questions:
(1) How robust are different submodular minimization algorithms, including PGM, to multiplicative noise? 
(2) How well can PGM minimize a non-submodular objective? Do the parameters $(\alpha, \beta)$ accurately characterize its performance?
%

All experiments were implemented in Matlab, and conducted on cluster nodes with 16 Intel Xeon E5 CPU cores and 64 GB RAM. Source code is available at \url{https://github.com/marwash25/non-sub-min}.

\subsection{Noisy submodular minimization} \label{sect:ExpNoisy}


First, we consider minimizing a submodular function $H$ given a noisy oracle  $\tH = \xi H$, where $\xi$ is independently drawn from a Gaussian distribution with mean one and standard deviation $0.1$. We evaluate the performance of different submodular minimization algorithms, on two example problems, minimum cut and clustering.
We 
use the Matlab code from \url{http://www.di.ens.fr/~fbach/submodular/}, and
compare seven algorithms: the minimum-norm-point algorithm (MNP) \cite{Fujishige2011}, the conditional gradient method \cite{Jaggi2013} with fixed step-size (CG-$2/(t+2)$) and with line search (CG-LS), PGM with fixed step-size 
(PGM-$1/\sqrt{t}$) and with the approximation of  Polyak's rule (PGM-polyak) \cite{Bertsekas1995}, the analytic center cutting plane method \cite{Goffin1993} (ACCPM) and a variant of it that emulates the simplicial method (ACCPM-Kelley). 

We replace the true oracle for $H$ by the approximate oracle $\tH_m(S) = \tfrac{1}{m} \sum_{i= 1}^m \xi_i H(S)$, for all these algorithms, and
test them on two datasets: \emph{Genrmf-long}, a min-cut/max-flow problem with $d = 575$ nodes and $2390$ edges, and \emph{Two-moons}, a synthetic semi-supervised clustering instance with $d = 400$ data points and $16$ labeled points. We refer the reader to  \citep[Sect. 12.1]{Bach2013} for more details about the algorithms and datasets. 
We stopped each algorithm after 1000 iterations for the first dataset and after 400 iterations for the second one, or until the approximate duality gap reached $10^{-8}$. To compute the optimal value $H(S^*)$, we use MNP with the noise-free oracle $H$. 

\begin{figure}
\begin{tabular}[width=\textwidth]{cc}%
\includegraphics[trim=15 210 45 230, clip, scale=.19]{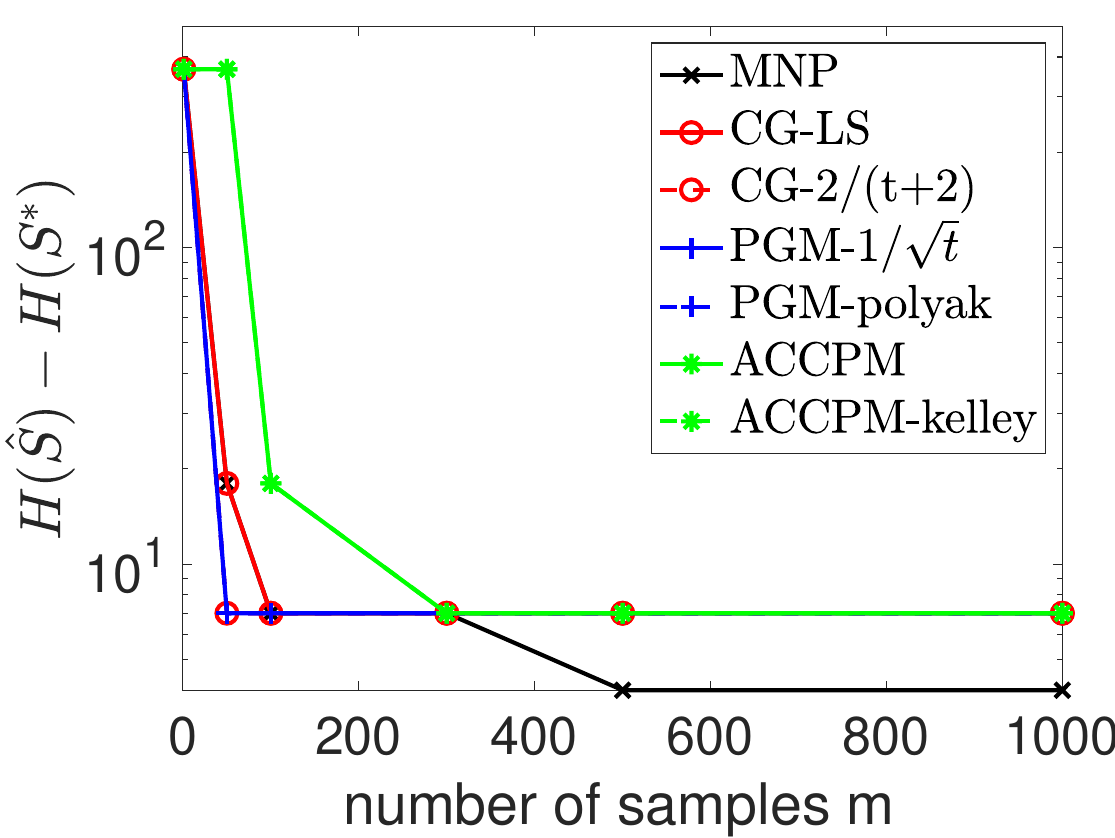} & \includegraphics[trim=15 210 45 230, clip, scale=.19]{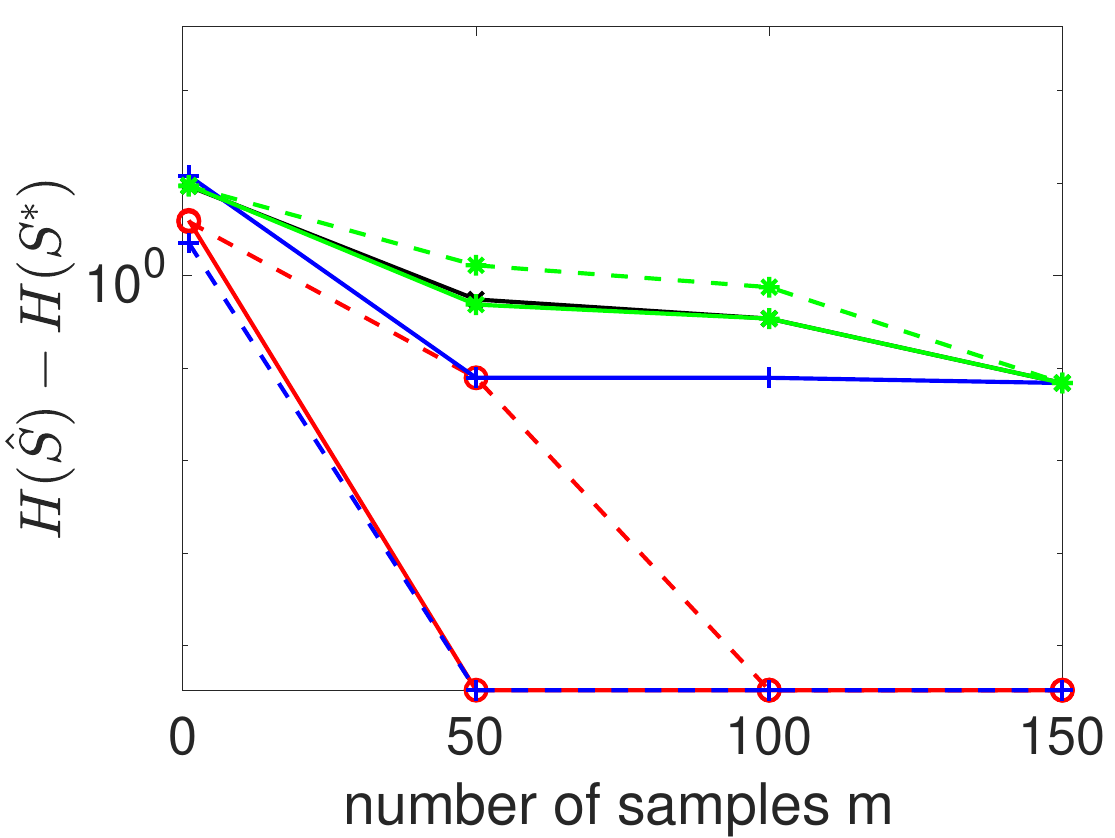}  \\
\includegraphics[trim=15 210 45 230, clip, scale=.19]{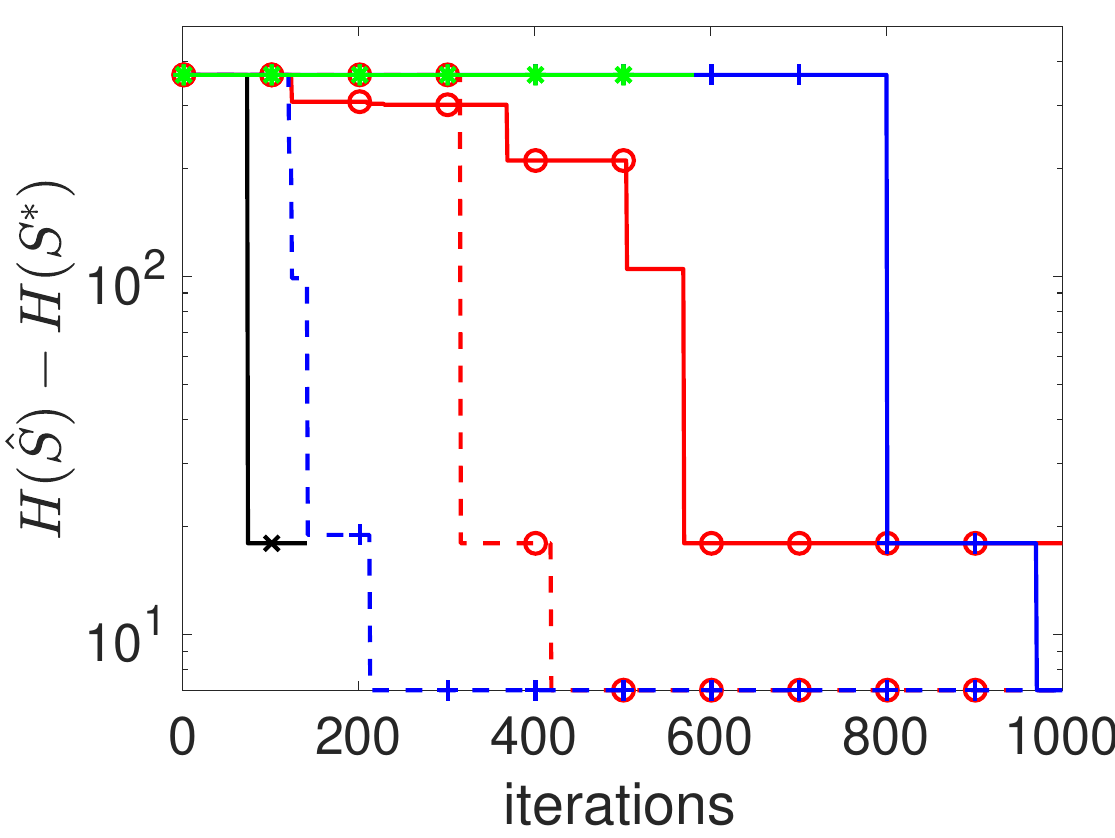} &
\includegraphics[trim=15 210 45 230, clip, scale=.19]{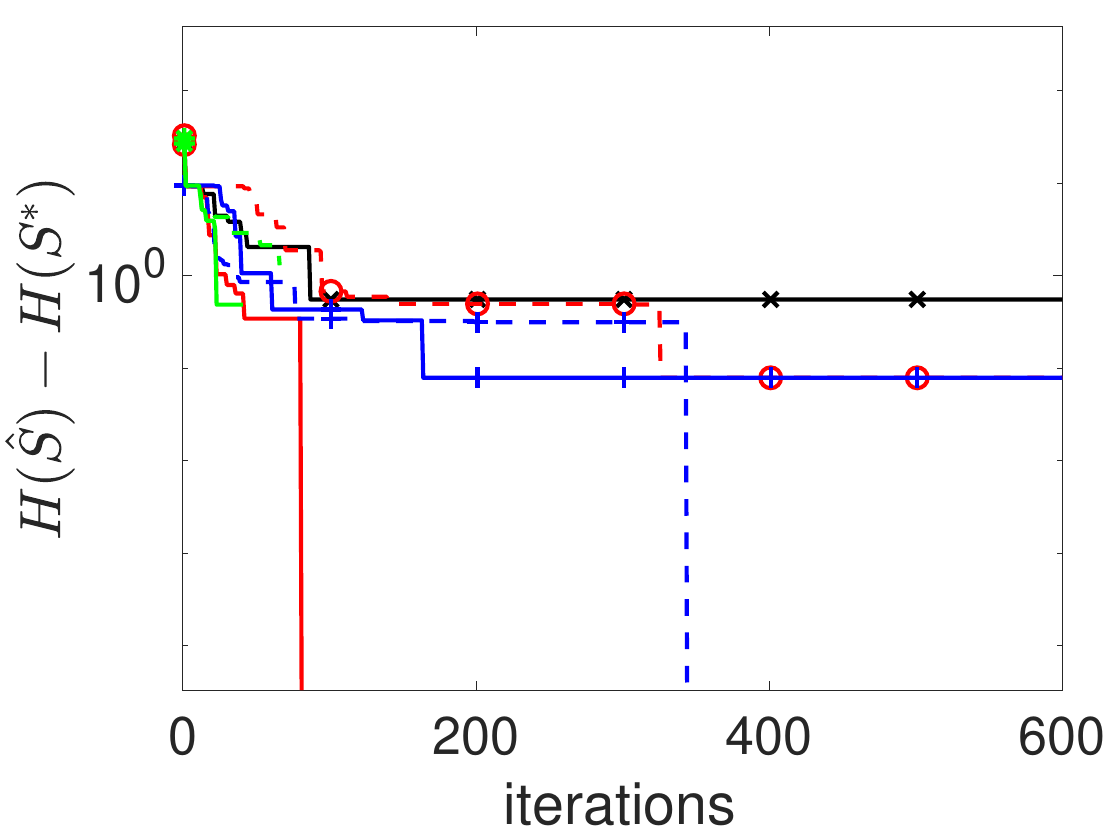} \\
\includegraphics[trim=15 210 45 230, clip, scale=.19]{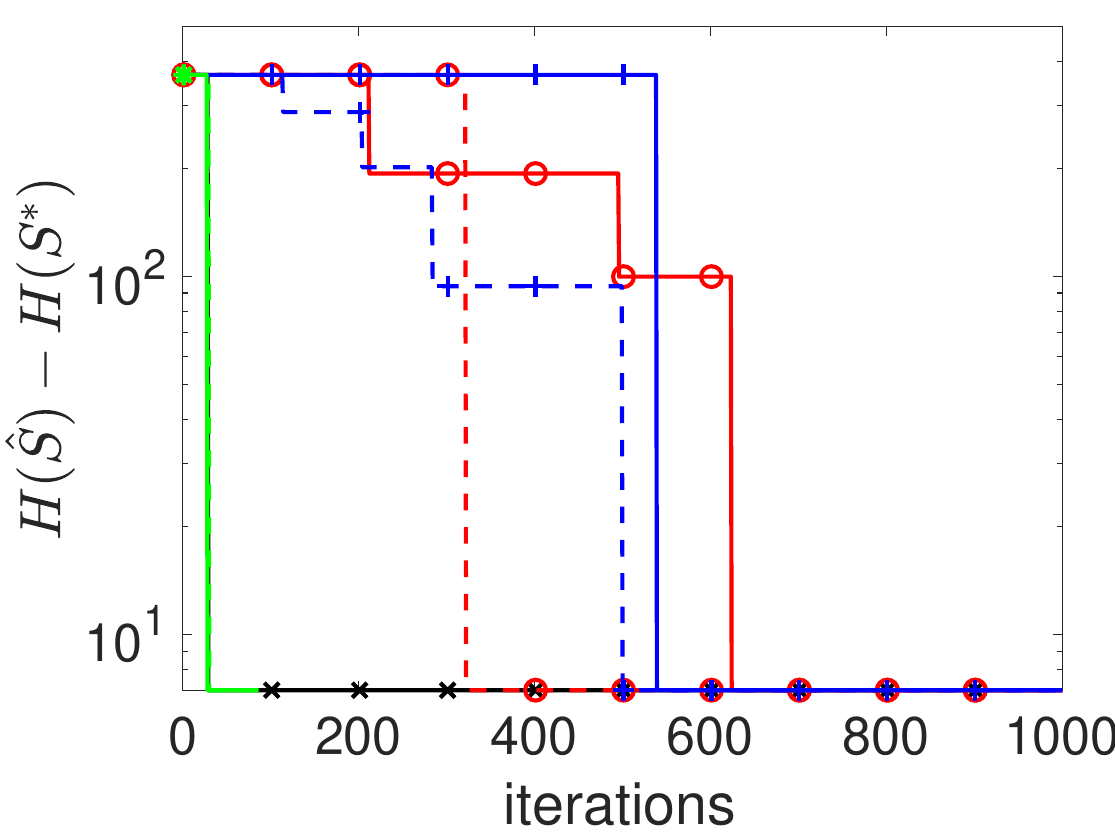} &
\includegraphics[trim=15 210 45 230, clip, scale=.19]{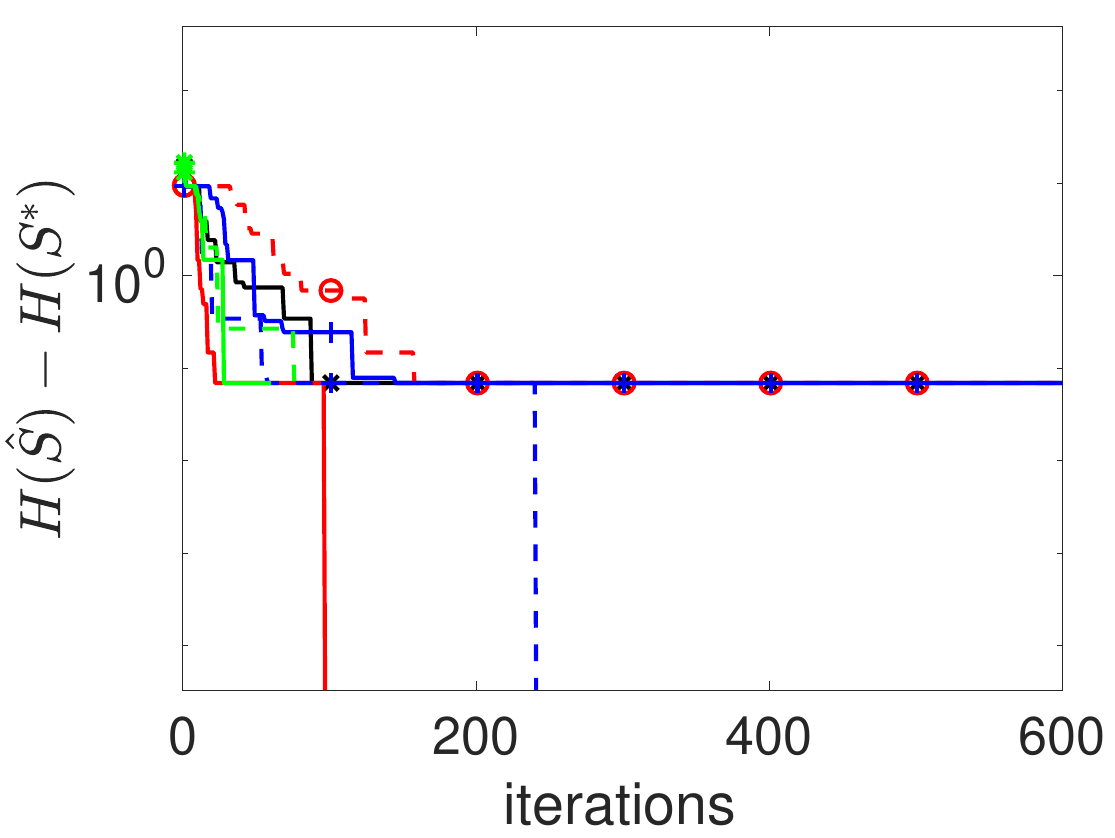} 
\end{tabular}
\caption{\label{fig:ResultsMoons}  Noisy submodular minimization results for \emph{Genrmf-long}  (right) and \emph{Two-moons} (left) data:  Best achieved objective (log-scale) vs.\ number of samples (top). Objective (log-scale) vs.\ iterations, for $m = 50$ (middle), $m = 1000$ (bottom-right), and $m = 150$ (bottom-left). }
\end{figure}

Figure \ref{fig:ResultsMoons} shows the gap in discrete objective value for all algorithms on the two datasets, for increasing number of samples $m$ (top), and for two fixed values of $m$, as a function of iterations (middle and bottom). We plot the best value achieved so far. As expected, the accuracy improves with more samples. In fact, this  improvement is faster than the bounds in Proposition \ref{prop:InconsistentNoiseSub} and in \cite{Blais2018}. The objective values in the \emph{Two-moons} data are smaller, which makes it easier to solve in the multiplicative noise setting (Prop. \ref{prop:InconsistentNoiseSub}), as we indeed observe.
 Among the compared algorithms, 
ACCPM and MNP converge fastest, as also observed in \cite{Bach2013} without noise, but they also seem to be the most sensitive to noise. 
In summary, these empirical results suggest that submodular minimization algorithms are indeed robust to noise, as predicted by our theory. 

\subsection{Structured sparse learning}\label{sect:ExpStructure}

Our second set of experiments is structured sparse learning,
where we aim to estimate a sparse parameter vector $\x^\natural \in \R^d$ whose support is an interval. 
The range function $F^r(S) = \max(S) - \min(S) + 1$ if $S \not = \emptyset$, and $F^r(\emptyset) = 0$, is a natural regularizer to choose.
$F^r$ is $\tfrac{1}{d-1}$-weakly DR-submodular \cite{ElHalabi2018}.
Another reasonable regularizer is the
  modified range function $F^{\mathrm{mr}}(S) = d - 1 + F^r(S), \forall S \not = \emptyset$ and $F^{\mathrm{mr}}(\emptyset) = 0$, which is non-decreasing and  submodular \cite{Bach2010}.
  As discussed in Section \ref{sect:StructSparse}, no prior method provides a guaranteed approximate solution to Problem \eqref{eq:structuredSparsity} with such regularizers, with the exception of 
  some statistical assumptions, under which $\x^\natural$ can be recovered using the tightest convex relaxation $\Theta^r$ of $F^r$ \cite{ElHalabi2018}.  Evaluating
  $\Theta^r$ involves a linear program with constraints corresponding to all possible interval sets. Such exhaustive search is not feasible in more complex settings.
  
 
 We consider a simple linear regression setting in which $\x^\natural \in \R^d$ has $k$ consecutive ones and is zero otherwise. We observe  $\y = \A \x^\natural + \boldsymbol{\epsilon}$, where $\A \in \R^{d \times n}$ is an i.i.d Gaussian matrix with normalized columns, and $\boldsymbol{\epsilon} \in \R^n$ is an i.i.d Gaussian noise vector with standard deviation $0.01$. We set 
$d = 250, k = 20$ and vary the number of measurements $n$ between $d/4$ and $2 d$.
 We compare the solutions obtained by minimizing the least squares loss $\ell(\x) = \frac{1}{2} \| \y - \A \x\|_2^2$ with the three regularizers:  
  The range function $F^r$,  where $H$ is optimized via exhaustive search (OPT-Range), or via 
PGM (PGM-Range);  the modified range function $F^{\mathrm{mr}}$, solved via exhaustive search (OPT-ModRange), or via PGM (PGM-ModRange); and   
 the convex relaxation $\Theta^r$ (CR-Range), solved using CVX \cite{Grant2014}. 
 The marginal gains of $G^\ell$ can be efficiently computed using rank-1 updates of the pseudo-inverse \cite{Meyer1973}. 

\begin{figure}
\begin{tabular}{c c}
 \includegraphics[trim=20 210 50 230, clip, scale=.19]{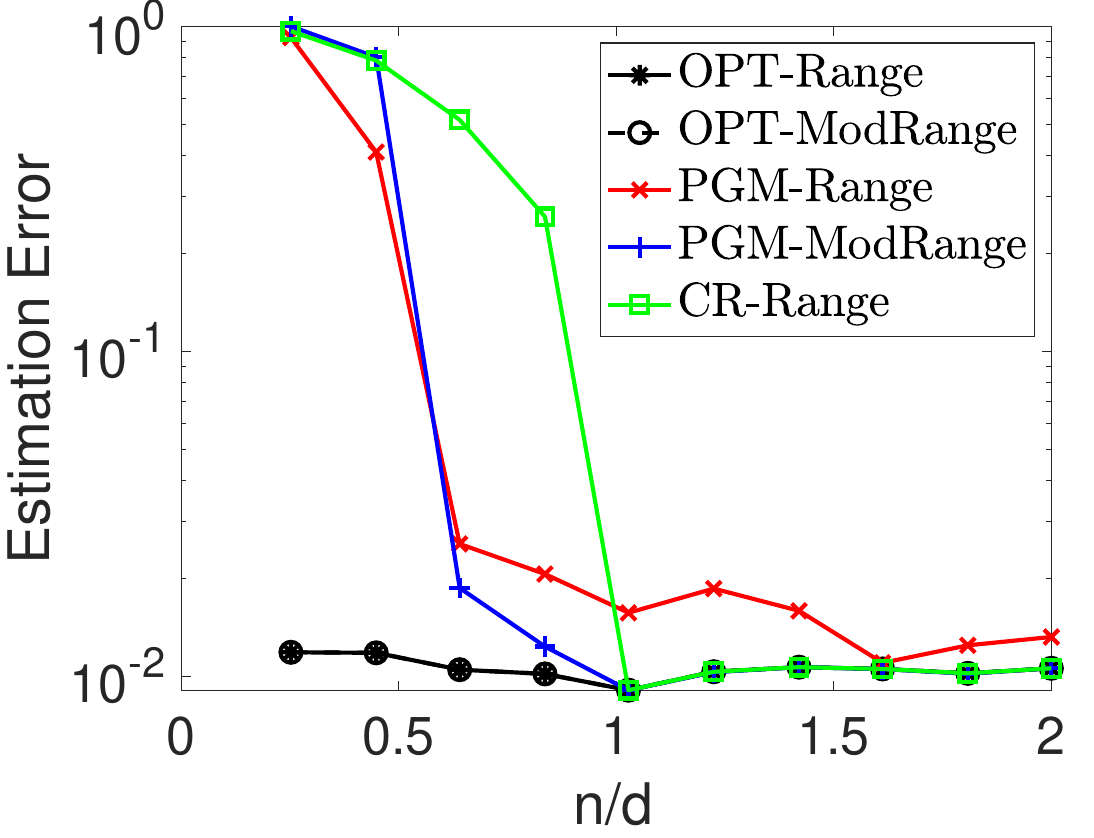} &
    \includegraphics[trim=20 210 50 230, clip, scale=.19]{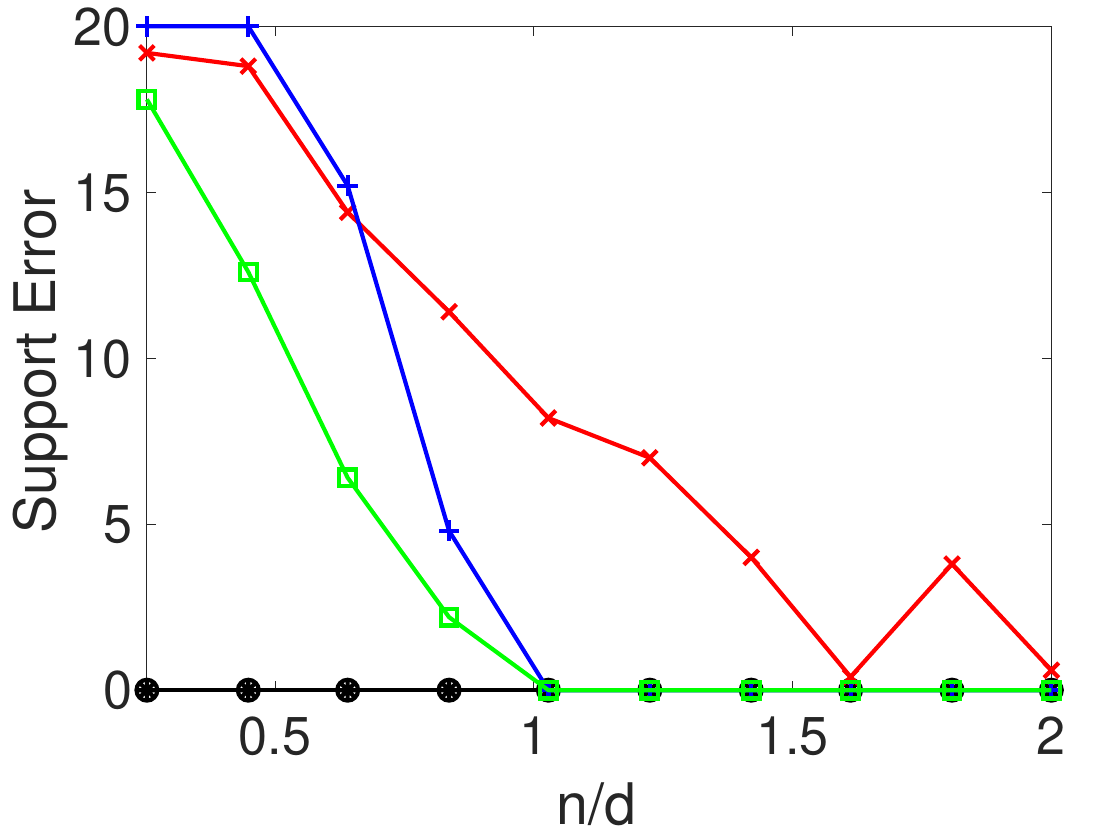} \\
 \includegraphics[trim=20 210 50 230, clip, scale=.19]{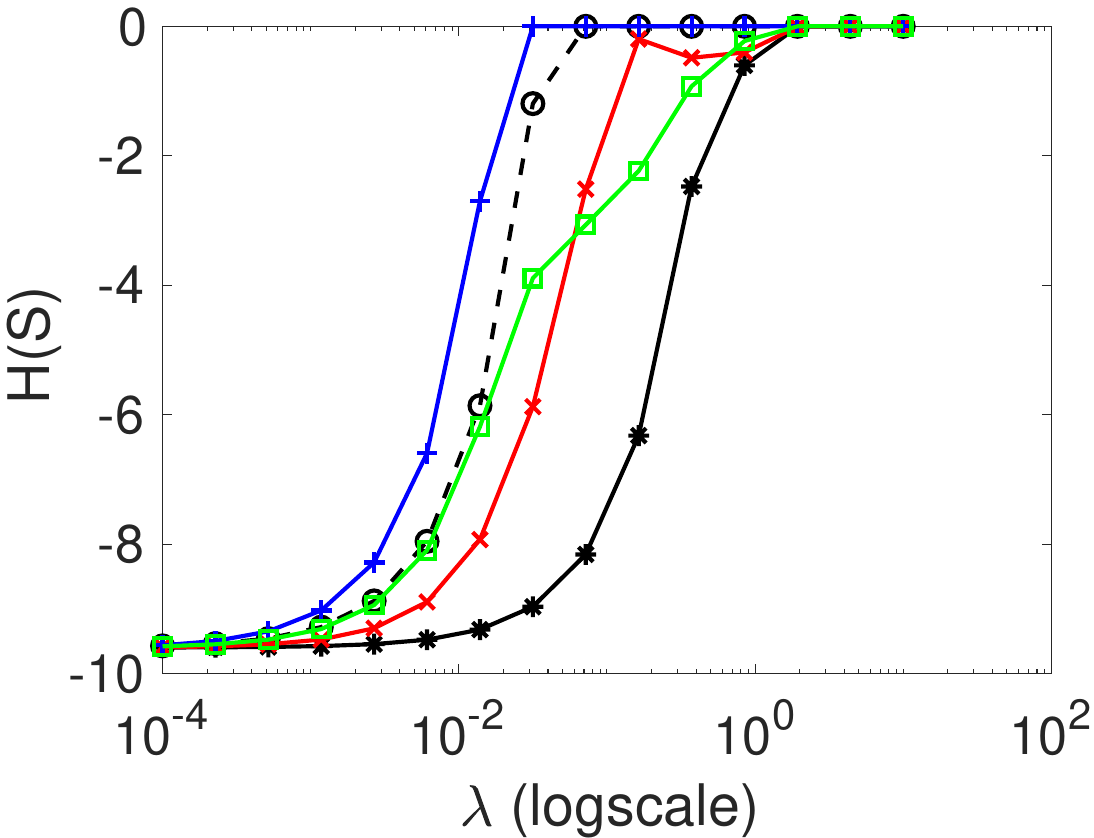} &
     \includegraphics[trim=20 210 50 230, clip, scale=.19]{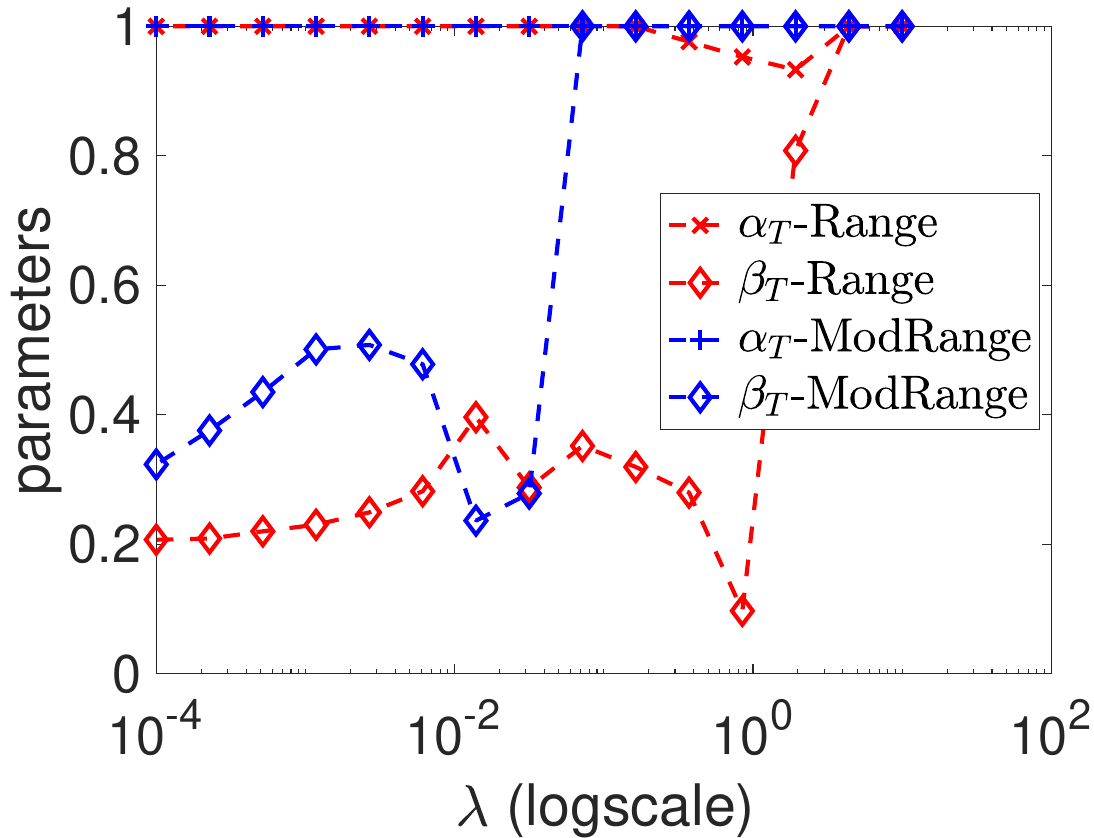}\\
\includegraphics[trim=20 210 50 230, clip, scale=.19]{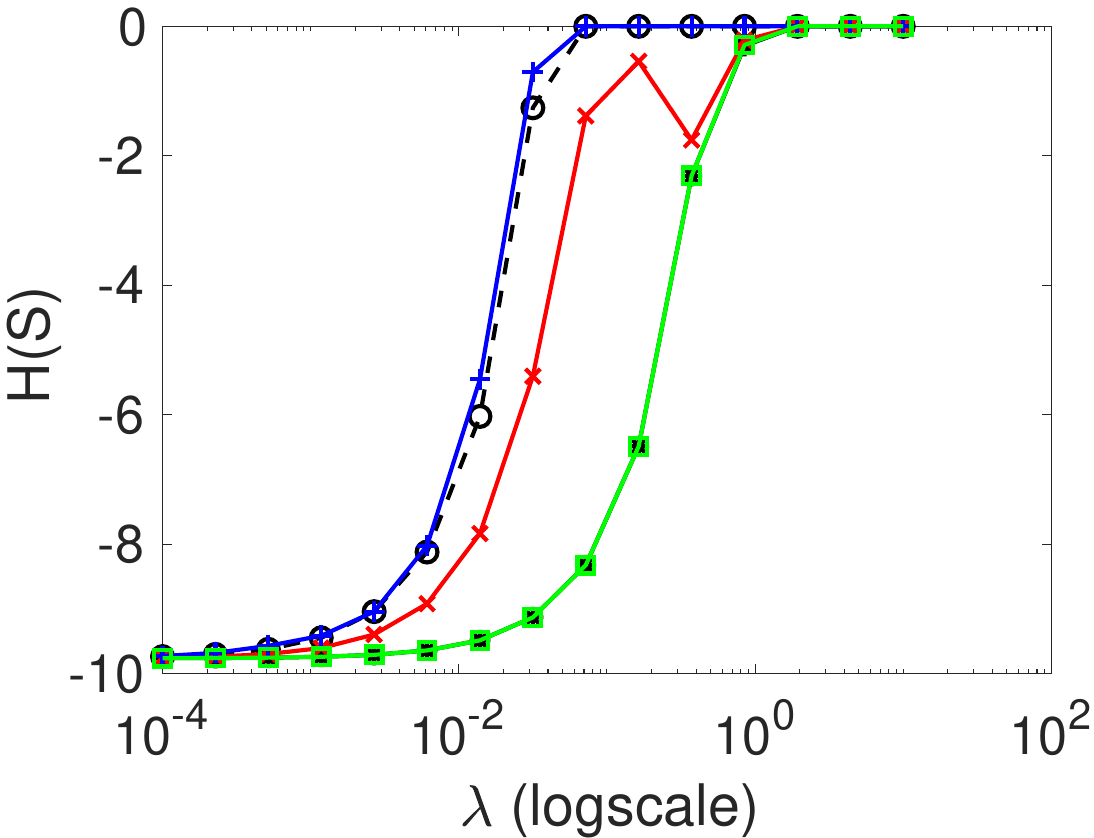}  &
 \includegraphics[trim=20 210 50 230, clip, scale=.19]{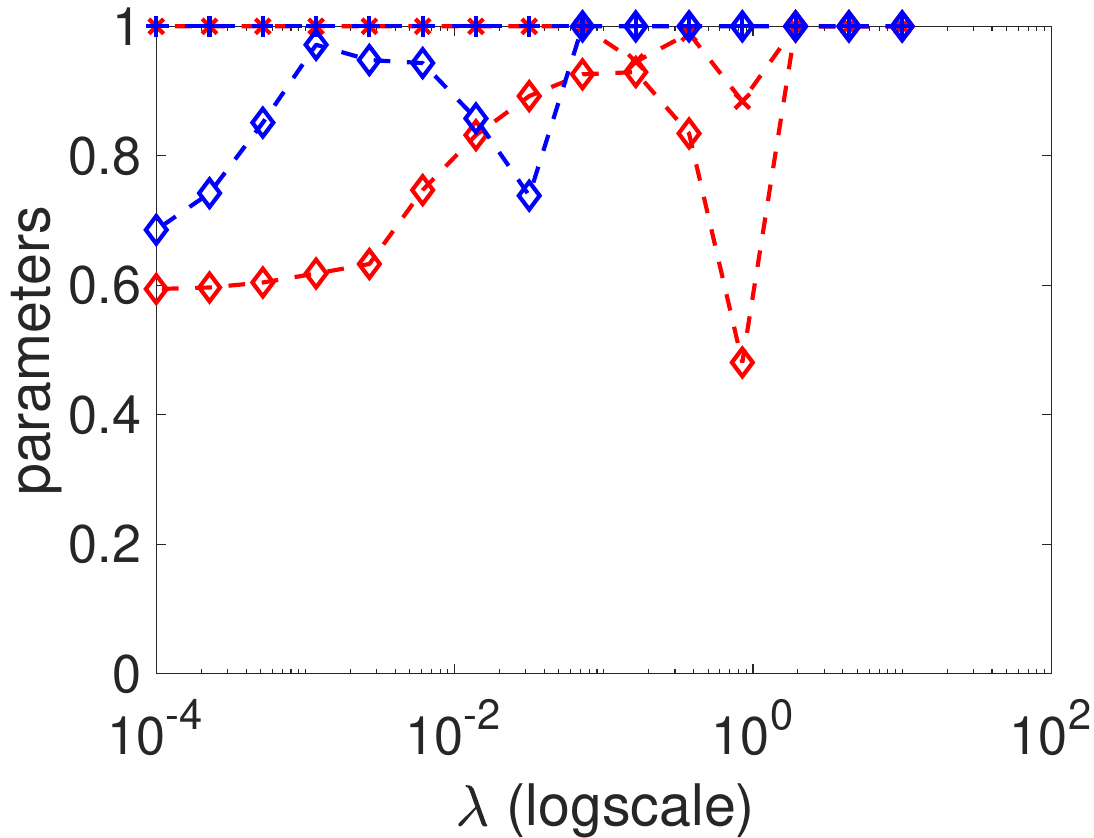}
\end{tabular}
\caption{\label{fig:RangeResults}  Structured sparsity results: Support and estimation errors  (log-scale) vs measurement ratio (top); objective and corresponding $\alpha_T, \beta_T$ parameters vs. regularization parameter for $n = 112$ (middle) and $n = 306$ (bottom).}
\vspace{-5pt}
\end{figure}

Figure  \ref{fig:RangeResults} (top) displays the best achieved support error in hamming distance, and estimation error $\| \hat{\x} - \x^\natural \|_2 / \|\x^\natural\|_2$ on the regularization path, where $\lambda$ was varied between $10^{-4}$ and $10$.
Figure \ref{fig:RangeResults} (middle and bottom) illustrates the objective value $H = \lambda F^r - G^\ell$ for PGM-Range, CR-Range, and OPT-Range, and $H = \lambda F^{mr} - G^\ell$ for PGM-ModRange, and OPT-ModRange, and the corresponding  parameters $\alpha_T, \beta_T$ defined in Remark \ref{rmk:parameters}, for two fixed values of $n$. Results are averaged over $5$ runs.

We observe that PGM minimizes the objective with $F^{mr}$ almost exactly as $n$ grows. It performs a bit worse with $F^r$, which is expected since $F^r$  is not submodular. This is also reflected in the support and estimation errors.
Moreover, $\alpha_T, \beta_T$ here reasonably predict the performance of PGM; larger values correlate with closer to optimal objective values.
They are also more accurate than the worst case $\alpha, \beta$ in Definition \ref{def:WDR}. Indeed, the $\alpha_T$ for the range function is  much larger than the worst case $\tfrac{1}{d-1}$. Similarly, $\beta_T$ for $G^\ell$ is  quite large and approaches 1 as $n$ grows, while in Proposition \ref{prop:LS-WDRmod} the worst case $\beta$ is only guaranteed to be non-zero when $\ell$ is strongly convex. 
Finally, 
the convex approach with $\Theta^r$  
 essentially matches the performance of OPT-Range when $n \geq d$. In this regime, $G^\ell$ becomes nearly modular, hence the convex objective $\ell + \lambda \Theta^r$ starts approximating the convex closure of $\lambda F^r - G^\ell$.


\section{Conclusion}
We established new links between approximate submodularity and convexity, and used them to analyze the performance of PGM for unconstrained, possibly noisy, non-submodular minimization. This yielded the first approximation guarantee for this problem, with  
a matching lower bound establishing its optimality.
We experimentally validated our theory, and illustrated the robustness of submodular minimization algorithms to noise and non-submodularity.

\section*{Acknowledgments}
This research was supported by a DARPA D3M award, NSF CAREER award 1553284, and NSF award 1717610. The views, opinions, and/or findings contained in this article are those of the authors and should not be interpreted as representing the official views or policies, either expressed or implied, of the Defense Advanced Research Projects Agency or the Department of Defense. The authors acknowledge the MIT SuperCloud and Lincoln Laboratory Supercomputing Center for providing HPC resources that have contributed to the research results reported within this paper.
\bibliographystyle{icml2020}
\bibliography{biblio}

%
%

\appendix
\onecolumn

\section{Extension to constrained minimization} \label{sect:ConstMin}
 
 Our result directly implies a generalization of some approximation guarantees of constrained submodular minimization to constrained weakly DR-submodular minimization. In particular, we consider the problem
 \begin{equation}\label{eq:ConstrMin}
 \min_{S \in \C} F(S),
 \end{equation}
 where $F$ is a monotone $\alpha$-weakly DR-submodular function and $\C$ denotes a family of feasible sets.
 We note that Theorem \ref{them:NonSubMin2} still holds in this setting, if we project the iterates onto the convex hull $\conv(\C)$ of $\C$.  We can thus obtain a solution $\hat{\s} \in \conv(\C)$ such that $f_L(\hat{\s}) \leq \frac{F(S^*)}{
\alpha} + \epsilon$ where $S^*$ is the optimal solution of \eqref{eq:ConstrMin}. However, the rounding in Corollary \ref{corr:thresholding} does not hold anymore, since not all sup-level sets of $\hat{\s}$ will be feasible.  
 
One rounding approach proposed in \cite{Iyer2014} is to simply pick the smallest feasible sup-level set. Given $\s \in [0,1]^d$, we pick the largest $\theta \in [0,1]$ such that $S_\theta = \{s_i : s_i \geq \theta \} \in \C$. 
The obtained set would then satisfy $F(S_\theta) \leq \frac{1}{\theta} f_L(\s)$. 
Applying this rounding to $\hat{\s}$, we obtain  $F(\hat{S}_\theta) \leq \frac{1}{\alpha \theta} F(S^*) + \epsilon$. 
In general there is no guarantee that $\theta \not = 0$. But for certain constraints, such as matroid, cut and set cover constraints, \citet{Iyer2014} show that $\theta$ admits non-zero bounds (see Table 2 in \cite{Iyer2014}). 

\section{Proofs for Section \ref{sect:Algo}}\label{sect:AppApproxProofs}

\primeDecomposition*
\begin{proof}
This decomposition builds on the decomposition of $H$ into the difference of two non-decreasing submodular functions  \cite{Iyer2012}. 
We start by choosing any function $G'$ which is non-decreasing $(\alpha,\beta)$-weakly DR-modular, and is strictly $\alpha$-weakly DR-submodular, i.e., $\epsilon_{G'}= \min_{i \in V, A \subset B \subseteq V \setminus i} G'(i|A) - \alpha G'(i|B) >0$. 
It is always possible to find such a function: for $\alpha = 1$, we provide an example in Proposition \ref{ex:(1,beta)-mod}. For $\alpha <1$, we can simply use $G'(S)  = |S|$. 
It is not possible to choose $G'$ such that $ \alpha = \beta = 1 $  (this would imply $G'(i|B) \geq G'(i|A) > G'(i|B)$). We  then construct $F$ and $G$ based on $G'$.

Let  $\epsilon_H = \min_{i \in V, A \subseteq B \subseteq V \setminus i} H(i|A)  - \alpha H(i|B) < 0$ the violation of $\alpha$-weak DR-submodularity of $H$; we may use a lower bound $\epsilon'_H \leq \epsilon_H$.
We define $$F'(S) =  H(S) + \frac{|\epsilon'_H|}{\epsilon_{G'}} G'(S),$$ then $F'(i|A) \geq \alpha F'(i|B), \forall i \in V, A \subset B \subseteq V \setminus i$, but not necessarily for $A = B$ since $F'$ is not necessarily non-decreasing. To correct for that, let  $V^- = \{ i : F'(i|V \setminus i) < 0\}$ 
and define $$F(S) = F'(S)-  \sum_{i \in S \cap V-} F'(i|V \setminus i).$$
For all $i \in V, A \subseteq V \setminus i$, if  $i \not \in V^-$ then $F(i|A)= F'(i|A)  \geq \alpha F'(i|V \setminus i) \geq 0$, otherwise $F(i|A) = F'(i|A) -  F'(i|V \setminus i) \geq (\alpha - 1) F'(i|V \setminus i) \geq 0$ for $A \not = V \setminus i$ and $F(i|V \setminus i) = 0$. $F$ is thus non-decreasing $\alpha$-weakly DR-submodular.
We also define $$G(S) =  \frac{|\epsilon'_H|}{\epsilon_{G'}} G'(S) -  \sum_{i \in S \cap V-} F'(i|V \setminus i),$$ then $H(S) = F(S) - G(S)$, and $G$ is non-decreasing $(\alpha, \beta)$-weakly DR-modular.
\end{proof}

\primePropExWeakMod*
\begin{proof}
$g$ is a concave function, since $a < 0$, hence $G'(S)$ is submodular. 
It also follows that 
\begin{align*}
\max_{i \in V, S \subseteq T \subseteq V\setminus i} \frac{G'(i|S)}{G'(i|T)} &= \max_{i \in V, S \subseteq T \subseteq V\setminus i} \frac{G'(i)}{G'(i|V \setminus i)} \\
&= \frac{\tfrac{1}{2} a + (1- \tfrac{1}{2} a) }{\tfrac{1}{2} a (d^2 - (d-1)^2) + (1- \tfrac{1}{2} a)(d - (d-1)) }\\
&= \frac{1}{\tfrac{1}{2} a (2 d -2) + 1}\\
&= \frac{1}{\beta}.
\end{align*}
We also have
\begin{align*}
\epsilon_{G'} &= \min_{i \in V, S \subset T \subseteq V\setminus i} G'(i|S) - G'(i|T)\\
&=  \min_{T \subset V} 2g(|T|) - g(|T|-1) - g(|T|+1)\\
&=    \min_{T \subset V}  \tfrac{1}{2} a (2 |T|^2 - (|T|-1)^2 - (|T|+1)^2) + (1- \tfrac{1}{2} a) (2 |T| - (|T|-1) - (|T|+1))\\
&=  - a . 
\end{align*}
\end{proof}

\subsection{Proofs for Section \ref{sect:ConvRel}}
\primeModApproxLemma*
\begin{proof}
Given any feasible point $(\kappabf',\rho')$ in the definition of $h^-$, i.e., $\kappabf(A) + \rho' \leq H(A), \forall A \subseteq V$, we have:
\begin{align*}
\kappabf^\top \s -  ( \kappabf'^\top \s + \rho') &= \sum_{k=1}^d s_{j_k} (H( j_k | S_{k-1}) -  \kappa'_{j_k}) -  \rho' \\
&= \sum_{k=1}^{d-1}  (s_{j_k} - s_{j_{k+1}}) \left(  H( S_{k}) -  \kappabf'(S_k) \right)   + s_{j_d} \left(   H( V) -  \kappabf'(V) \right)  -  \rho' \\
&\geq \left(\sum_{k=1}^{d-1}  (s_{j_k} - s_{j_{k+1}})  + s_{j_d} \right)  \rho'  -  \rho' \\
&= (s_{j_1} - 1)  \rho' \geq 0
\end{align*}
Hence $\kappabf^\top \s \geq h^-(\s)$. 
The last inequality holds by noting that $\rho' \leq 0$ since $\kappabf(\emptyset) + \rho' \leq H(\emptyset) = 0$.

The upper bound on $\kappabf(A)$ for any $A \subseteq V$ follows 
from the definition of weak DR-submodularity.
\begin{align*}
 \kappabf(A) &= \sum_{j_k \in A}  H( j_k | S_{k-1}) \\
&\leq \sum_{j_k \in A} \frac{1}{\alpha} F( j_k | A \cap S_{k-1}) - \beta G( j_k | A \cap S_{k-1}) \\
&=  \sum_{k=1}^d \frac{1}{\alpha} (F(A \cap S_{k}) - F(A \cap S_{k-1})) - \beta  (G(A \cap S_{k}) - G(A \cap S_{k-1})) \\ 
&= \frac{F(A)}{\alpha} - \beta G(A)
\end{align*}
Note that $\kappabf$ can be written as $\kappabf= \kappabf^F - \kappabf^G$ where $\kappabf^F_{j_k} = F( j_k | S_{k-1}) $ and $\kappabf^G_{j_k} = G( j_k | S_{k-1})$. We have $ \kappabf^F(A) \leq \frac{F(A)}{\alpha}$ and $\kappabf^G(A) \leq  \beta G(A), \forall A \subseteq V$. Hence $(\alpha \kappabf^F, \0)$ and $(\frac{1}{\beta}\kappabf^G, \0)$ are feasible points in the definitions of $f^-$ and $(-g)^-$. 
The bound on $\kappabf^\top \s'$ for any $\s' \in [0,1]^d$ then follows directly from the definitions of $f^-$ and $(-g)^-$ \eqref{eq:ClosureMaxForm}.
\end{proof}

\subsection{Proofs for Section \ref{sect:AlgoGaurantee}}

\primeNonSubTheom*
\begin{proof}
Let $\z^{t+1} = \s^t - \eta \kappabf^t$, then note that $\| \s^{t+1} - \s^*\|_2 \leq \| \z^{t+1} - \s^*\|_2$ due to the properties of projection (see for e.g., \citep[Lemma 3.1]{Bubeck2014}), it follows then
\begin{align*}
\langle \kappabf^t, \s^t  - \s^* \rangle &= \frac{1}{\eta} \langle \s^t  - \z^{t+1} , \s^t  - \s^*  \rangle \\
&= \frac{1}{2 \eta} ( \| \s^t  - \z^{t+1} \|_2^2 + \| \s^t  - \s^* \|_2^2 - \| \z^{t+1}  - \s^* \|_2^2 ) \\ 
&=\frac{1}{2 \eta} (  \| \s^t  - \s^* \|_2^2 - \| \z^{t+1}  - \s^* \|_2^2 ) + \frac{\eta}{2 } \| \kappabf^t\|_2^2\\
&\leq \frac{1}{2 \eta} (  \| \s^t  - \s^* \|_2^2 - \|\s^{t+1}  - \s^* \|_2^2 ) + \frac{\eta}{2 } \| \kappabf^t\|_2^2
\end{align*}

Summing over $t$ we get
\begin{align*}
\sum_{t=1}^T \langle \kappabf^t, \s^t  - \s^* \rangle &\leq T\frac{R^2}{2 \eta} + \frac{\eta T L^2}{2} 
\end{align*}
Since $F$ is $\alpha$-weakly DR submodular and $-G$ is $\frac{1}{\beta}$-weakly DR submodular, we have by lemma \ref{lem:ModularApprox}  for all $t$, $(\kappabf^t)^\top \s^* \leq  \frac{f^-(\s^*)}{\alpha} + \beta (-g)^-(\s^*)$ and $ (\kappabf^t)^\top \s^t = h_L(\s^t) \geq h^-(\s^t)$.
 Plugging in the value of $\eta$, we thus obtain
$$ \min_{t}   h^-(\s^t) \leq  \min_{t}   h_L(\s^t)  \leq  \frac{f^-(\s^*)}{\alpha} + \beta (-g)^-(\s^*)  + \frac{R L}{\sqrt{T}}.$$
\end{proof}

\primeCorrThresholding*
\begin{proof} By definition of the Lov\'asz extension, 
$
h_L(\hat{\s}) = \sum_{k=1}^{d-1}  (\hat{s}_{j_k} - \hat{s}_{j_{k+1}}) H( \hat{S}_{k}) + \hat{s}_{j_d} H(V) \geq \min_{k \in \{0,\cdots, d\}}  H(\hat{S}_k). 
$
The corollary then  follows by Theorem \ref{them:NonSubMin2} and the extension property $f^-(\s^*) = F(S^*), (-g)^-(\s^*) = -G(S^*)$. 
\end{proof}

\primeNonincreasing
\begin{proof}
We may write $\tH(S) = \tF(S) - \tG(S)$, where  $\tF(S) = F(V \setminus S)$ is non-decreasing $1/\alpha$-weakly DR-submodular, and $\tG(S) = G(V\setminus S)$ is non-decreasing $1/\beta$-weakly DR-supermodular.  Let $\tilde{S}^* \in \argmin_{S \subseteq V}  \tH(S)$, then $S^* = V \setminus \tilde{S}^*$.  Corollary \ref{corr:thresholding} implies that $H(\hat{S}) = \tH(  \tilde{S}) \leq {\alpha} \tF( \tilde{S}^*) - \tfrac{1}{\beta} \tG( \tilde{S}^*)=  {\alpha} F(S^*) - \tfrac{1}{\beta} G(S^*)$. The tightness follows from  Proposition \ref{ex:tightnessResult}, too.
\end{proof}

\subsection{Proofs for Section \ref{sect:NoisySFM}}
\primeapproximateOracle*
\begin{proof}
We define $\kappabf$ as $\kappa_{j_k} = H(j_k | S_{k-1})$ and $\tilde{\kappabf}$ as  $\tkappa_{j_k} = \tH(j_k | S_{k-1})$ for any ordering on $V$.
For all $k \in V$, we have $|\tkappa_{j_k} - \kappa_{j_k}| \leq 2 \epsilon$ with probability $1 - 2 d \delta$ (by a union bound), hence $|\boldsymbol{\tkappa}(S^*) - \kappabf(S^*)| \leq 2\epsilon |S^*|$.  Plugging this into the proofs of Theorem \ref{them:NonSubMin2} and Corollary \ref{corr:thresholding}  yields $ H(\hat{S}) \leq  \tfrac{1}{\alpha} F(S^*)- \beta G(S^*) + 2 \epsilon (|S^*| + 1) + \tfrac{R L}{\sqrt{T}},$
with probability $1 - 2 d T \delta$ (by a union bound).
The proposition follows by setting $\epsilon, \delta$ and $T$ to the chosen values.
\end{proof}

\primeInconsistentNoiseSub*
 \begin{proof}
 For every $S \subseteq V$ and $\epsilon' > 0$, a Chernoff bound implies that
$(1 - \epsilon') \mu H(S) \leq \tH_m(S) = \tfrac{1}{m} \sum_{i=1}^m \xi_i H(S) \leq (1 + \epsilon') \mu H(S),$
with probability at least $1 - \exp(- \tfrac{\epsilon'^2 \mu^2 m}{\omega^2})$.
Choosing $\epsilon' = \tfrac{\epsilon}{\mu H_{\max}}$  yields the proposition.
 \end{proof}

 \subsection{Proofs for Section \ref{sec:lowerbd}} \label{sect:AppLowerBd}

\primeThemLowerBd*
\begin{proof}
Let $C, D$ be two sets that partition the ground set $V = C \cup D$ such that $|C| = |D|= d/2$. We construct a normalized set function $H$ whose values depend only on $k(S) = |S \cap C|$ and $\ell(S)= |S \cap D|$. In particular, we define 
\begin{align*}
H(S) = \begin{cases}
0 &\text{if $|k(S) - \ell(S)| \leq \epsilon d$}\\
\frac{2 \alpha \delta}{ 2 - d }&\text{otherwise}
\end{cases},
\end{align*}
for some $\epsilon \in [1/d, 1/2)$.
By Proposition \ref{prop:Decomposition-alpha-gen}, given a non-decreasing $(\alpha,\beta)$-weakly DR-modular function $G'$, we can write $H(S) = F(S) - G(S)$, where $F(S) = H(S) + \frac{|\epsilon_H|}{\epsilon_{G'}} G'(S)$ is normalized non-decreasing $\alpha$-weakly DR-submodular, and $G(S) = \frac{|\epsilon_H|}{\epsilon_{G'}} G'(S)$ is normalized non-decreasing $(\alpha,\beta)$-weakly DR-modular. 
Note that $V^- = \emptyset$ in this case, since $H(i|V\setminus i) = 0$.
We choose $G'(S) = |S|$ if $\alpha <1$, then $\epsilon_{G'} = \min_{i \in V, S \subset T \subseteq T \setminus i} G(i|S) - \alpha G(i|T)  = 1 - \alpha > 0$. If $\alpha = 1$, we use the $(1, \beta)$-weakly DR-modular function defined in Proposition \ref{ex:(1,beta)-mod}, then $\epsilon_{G'} = \frac{1 - \beta}{d-1} > 0$.

Let the partition $(C, D)$ be random and unknown to the algorithm. 
We argue that, with high probability, any given query $S$ will be ``balanced'', i.e., $|k(S) - \ell(S)| \leq \epsilon d$. Hence no deterministic algorithm can distinguish between $H$ and the constant zero function. Given a fixed $S \subseteq V$, let $X_i = 1$ if $i \in C$ and $0$ otherwise, for all $i \in S$, then $\mu = \Eb[\sum_{i \in S} X_i] = \sum_{i \in S} \frac{|C|}{d} = \frac{|S|}{2}$. Then by a Chernoff's bound we have $Pr(|k(S) - \ell(S)| > \epsilon d)  \leq 2 \exp(-\frac{\epsilon^2 d}{4})$.
Hence, given a sequence of $e^{\frac{\epsilon^2 d}{8}}$ queries, the probability that  each query $S$ is balanced, and thus has value $H(S) = 0$, is still at least $1 - 2 e^{-\frac{\epsilon^2 d}{8}}$. On the other hand, we have $H(S^*) = \frac{2 \alpha \delta}{ 2 - d } <0$, achieved at $S^* = C$ or $D$.
Moreover, note that $\epsilon_H = \min_{i \in V, S \subseteq T \subseteq T \setminus i} H(i|S) - \alpha H(i|T)  = (1+\alpha) H(S^*)$.
Hence $$\tfrac{1}{\alpha}F(S^*) - \beta G(S^*) - \delta = \tfrac{1}{\alpha} H(S^*) \biggl(  1 -  (1- \alpha \beta) (1+\alpha) \frac{G'(S^*)}{\epsilon_{G'}}  \biggl) - \delta < 0,$$
since $\tfrac{G'(S^*)}{\epsilon_{G'}} = \tfrac{d}{2(1- \alpha)}$ if $\alpha < 1$, and $\frac{G'(S^*)}{\epsilon_{G'}} \geq  \frac{3 d (d-1)}{8 (1- \beta)}$, if $\alpha = 1$.

Therefore, with high probability, the algorithm cannot find a set with value $H(S) \leq \tfrac{1}{\alpha} F(S^*) - \beta G(S^*) - \delta$. This also holds for a randomized algorithm, by averaging over its random choices. 
\end{proof}

\begin{restatable}[Tight example for PGM]{proposition}{primeTightExample}\label{ex:tightnessResult}
For any $\alpha, \beta \in (0,1]$, there exists a set function $H(S) = F(S)- G(S)$, where $F$ is a non-decreasing $(\alpha, 1)$-weakly DR-modular function  and $G$ is a non-decreasing $(1,\beta)$-weakly DR-modular function, such that  the solution $\hat{S}$ in Corollary \ref{corr:thresholding} satisfies $H(\hat{S})  = F(S^*)/\alpha - \beta G(S^*).$
\end{restatable}
\begin{proof}
We define $$F(S) = \begin{cases} |S| + \tfrac{d}{\beta} -1 & \text{if $1 \in S$} \\
\alpha |S| &\text{ otherwise}\end{cases}, \text{ and } G(S) = \begin{cases} |S| + \tfrac{d}{\beta} -1 & \text{if $1 \in S$} \\
 \tfrac{1}{\beta} |S| &\text{ otherwise}\end{cases}.$$
It's easy to see that both $F$ and $G$ are monotone functions. For all $A \subseteq B, i \in V \setminus B$, we have 
\begin{align*}
\frac{F(i | A)}{F(i | B)} = \begin{cases}
1 &\text{if $1 \in A$ or $1 \not \in B$} \\
\alpha &\text{if $1 \not \in A, 1 \in B$}\\
\frac{ \frac{d}{\beta} + (1 - \alpha) |A|}{\frac{d}{\beta} + (1 - \alpha) |B|} &\text{if $i = 1$}\\
\end{cases}
\end{align*}
Note that $\frac{d}{\beta} + (1 - \alpha) |A| \geq \frac{d}{\beta} \geq \alpha (\frac{d}{\beta} + (1 - \alpha) |B|)$, hence $\alpha \leq \frac{F(i | A)}{F(i | B)} \leq 1$, which proves that $F$ is supermodular and $\alpha$-weakly DR-submodular.

Similarly we have  
\begin{align*}
\frac{F(i | A)}{F(i | B)} = \begin{cases}
1 &\text{if $1 \in A$ or $1 \not \in B$} \\
\frac{1}{\beta} &\text{if $1 \not \in A, 1 \in B$}\\
\frac{ \frac{d}{\beta} + (1 - \frac{1}{\beta}) |A|}{\frac{d}{\beta} + (1 - \frac{1}{\beta}) |B|} &\text{if $i = 1$}\\
\end{cases}
\end{align*}
Note that $\frac{d}{\beta} + (1 - \frac{1}{\beta}) |A| \leq \frac{d}{\beta} \leq \frac{1}{\beta} (\frac{d}{\beta} + (1 - \frac{1}{\beta}) |B|)$, hence $1 \leq \frac{F(i | A)}{F(i | B)} \leq \frac{1}{\beta}$, which proves that $G$ is monotone submodular and $\beta$-weakly DR-supermodular.

It remains to show that the solution obtained by projected subgradient method and thresholding have value $H(\hat{S}) = 0$. We can assume w.l.o.g that the starting point $\s^1$ is such that the largest element is $j_1 = 1$ (otherwise we can modify the example to have whatever is the largest element as the ``bad element''). Note that $H(1) = H(j_k | S_k) = 0$ for all $k \in [d]$, hence $\kappabf^1 = \mathbf{0}$ and $\s^t = \s^1$  and $\kappabf^t = \mathbf{0}$ for all $t \in \{1, \cdots, T\}$. Thresholding $\s^1$ would thus yield $H(\hat{S})  = 0$, with $\hat{S}  = \emptyset$ or any other set such that $1 \in \hat{S}$.
\end{proof}

\section{Proofs for Section \ref{sect:Application}}
\subsection{Proofs for Section \ref{sect:StructSparse}} \label{sect:AppApplicationProofs}
We actually prove Proposition \ref{prop:LS-WDRmod} under a more general setting: if $\ell$ has $\nu$-restricted smoothness (RSM) and $\mu$-restricted strong convexity (RSC) on the domain of $k$-sparse vectors, $G^\ell$ is weakly DR-modular for all sets of cardinality $k$. 
  Our current algorithm analysis requires weak DR-submodularity to hold for all sets.  Whether the algorithm can be modified to only query sets of cardinality $k$ 
   is an interesting question for future work.  %
 Let's recall  the definition of RSC/RSM.\\

\begin{definition}[RSM/RSC]
Given a differentiable function $\ell: \R^d \to \R$ and $\Omega \subset \R^d \times \R^d$, $\ell$ is $\mu_\Omega$-RSC and $\nu_\Omega$-RSM if 
$\frac{\mu_\Omega}{2} \| \x - \y \|_2^2 \leq \ell(\y) - \ell(\x) - \langle \nabla \ell(\x), \y - \x \rangle \leq \frac{\nu_\Omega}{2} \| \x - \y \|_2^2,\;\; \forall (\x,\y) \in \Omega.$
\end{definition}

If $\ell$ is RSC/RSM on $\Omega = \{(\x,\y): \|\x\|_0 \leq k_1, \| \y \|_0 \leq k_1, \| \x - \y \|_0 \leq k_2\}$, we denote by $\mu_{k_1, k_2}, \nu_{k_1, k_2}$ the corresponding RSC and RSM parameters. For simplicity, we also define $\mu_{k} := \mu_{k, k}, \nu_{k} := \mu_{k, k}$. 

 Before we can prove Proposition \ref{prop:LS-WDRmod}, we need two key lemmas.
Lemma \ref{lem:LSmarginals} restates a result from \cite{Elenberg2018}, 
which relates the marginal gain of $G^\ell$ to the marginal decrease in $\ell$. 
In Lemma \ref{lem:DenseOptSupp}, 
we argue that for a class of loss functions, namely RSC/RSM functions of the form  $\ell(\x) = L(\x) - \z^\top \x$, where $\z$ is a random vector, the corresponding minimizer has full support  with probability one. Proposition \ref{prop:LS-WDRmod-app} then follows from these two lemmas by noting that $\ell$ thus have non-zero marginal decrease, with respect to any $i \in V$, with probability one.\\

\begin{lemma}\label{lem:LSmarginals} Given  $G^\ell(S) = \ell(0) - \min_{\supp(\x) \subseteq S} \ell(\x)$, then
for any disjoint sets $A, B \subseteq V$ and a corresponding minimizer $\x^A := \argmin_{\supp(\x) \subseteq A} \ell(\x)$, if $\ell$ is $\mu_{|A \cup B|}$-RSC and $\nu_{|A|, |B|}$-RSM, we have:
\[ \frac{\| [\nabla \ell(\x^A)]_B\|_2^2}{2 \nu_{|A \cup B|, |B|}} \leq G^\ell(B | A) \leq \frac{\| [\nabla \ell(\x^A)]_B\|_2^2}{2 \mu_{|A \cup B|}}  \] 
\end{lemma}

\begin{lemma}\label{lem:DenseOptSupp}
If $\x^\star$ is the minimizer of $\min_{x \in \R^d} L(\x) - \z^\top \x$, where $L$ is a strongly-convex and smooth loss function, and $\z \in \R^d$ has a continuous density w.r.t to the Lebesgue measure, then $\x^\star$ has full support with probability one.
\end{lemma}
\begin{proof}
This follows directly from \citep[Theorem 1]{ElHalabi2018} by taking $\Phi(x) = 0$. We include the proof here for completeness.\\

Since $L$ is strongly-convex, given $\z$ the corresponding minimizer $\x^\star$ is unique, then the function $E(\z) := \argmin_{\x \in \R^d} L(\x) - \z^T \x$ is well defined.
Given fixed $i \in V$, we show that the set $S_i := \{ \z :  [E(\z)]_i = 0 \}$ has measure zero. Then, taking the union of the finitely many sets $S_i, \forall i \in V$, all of zero measure, we have $P(\exists \z \in \R^d,  \exists i \in V, \text{ s.t., } [E(\z)]_i = 0 ) = 0$.\\

To show that the set $S_i$ has measure zero, let $\z_1, \z_2 \in S_i$ and denote by $\mu>0$ the strong convexity constant of $L$. We have by optimality conditions:
$$\Big( \big( \z_1 - \nabla L(E(\z_1)) \big) - \big( \z_2 - \nabla L(E(\z_2)) \big) \Big)^\top \Big( E(\z_1)- E(\z_2)\Big) = 0$$
Hence,
\begin{align*}
(\z_1 - \z_2)^\top(E(\z_1)- E(\z_2))  &\geq  \big( \nabla L(E(\z_1))  - \nabla L(E(\z_2)) \big)^\top \big( E(\z_1)- E(\z_2) \big) \\
(\z_1 - \z_2)^\top(E(\z_1)- E(\z_2))  &\geq  \mu \| E(\z_1)- E(\z_2)\|_2^2 \\
\frac{1}{\mu} \|\z_1 - \z_2\|_2  &\geq \| E(\z_1)- E(\z_2)\|_2
\end{align*}
Thus $E$ is a deterministic Lipschitz-continuous function of $\z$. 
By optimality conditions $\z = \nabla L(E(\z))$, then $z_i = \nabla L(E(\z_{V \setminus i}))_i$.
Thus $z_i$ is a Lipschitz-continuous function of $\z_{V \setminus i}$, which can only happen with zero measure.
\end{proof}

\begin{proposition}\label{prop:LS-WDRmod-app}
Let  
$\ell(\x) = L(\x) - \z^\top \x$, where $L$ is $\mu_{|U|}$-RSC and $\nu_{|U|}$-RSM for some $U \subseteq V$ 
and $\z \in \R^d$ has a continuous density w.r.t  the Lebesgue measure. Then there exist $\alpha_G, \beta_G > 0$ such that $G^\ell$ is $(\alpha_G,\beta_G)$-weakly DR modular on $U$ (i.e., Def. \ref{def:WDR} restricted to sets $A \subseteq B \subseteq U$).
\end{proposition}
\begin{proof}
Given $ i \in U, A \subseteq U \setminus i$, let $\x^A := \argmin_{\supp(\x) \subseteq A} \ell(\x)$, then by Lemma \ref{lem:LSmarginals} and $\nu_{|A|+ 1,1} \leq \nu_{|A|+ 1}$ we have:
\begin{align*}
 \frac{ [\nabla \ell(\x^A)]_i^2}{2 \nu_{|A|+ 1}} &\leq G^\ell(i | A) \leq \frac{ [\nabla \ell(\x^A)]_i^2}{2 \mu_{|A| + 1}} 
\end{align*}
We argue that $[\nabla \ell(x^A)]_i^2 \not = 0$ with probability one. For that, we define $\ell_A(\uu):= \ell(\x)$, where $[\x]_A = \uu, [\x]_{V\setminus A} = 0, \forall \uu \in \R^{|A|}$, then  $\ell_A$ is $\mu_{|A|}$-strongly convex and $\nu_{|A|}$-smooth on $\R^{|A|}$. Hence, by lemma \ref{lem:DenseOptSupp}, the minimizer $\uu^\star$ of $\ell_A$ has full support with probability one, and thus $\supp(\x^A) = A$ also with probability one. By the same argument, we have $\supp(\x^{A \cup \{i\}}) = A \cup \{i\}$. We can thus deduce that  $[\nabla \ell(x^A)]_i^2 \not = 0$, since otherwise $G^\ell(i | A) = 0$, which implies that $\x^{A \cup \{i\}}  = x^A$ (minimizer is unique) and $\supp(\x^{A \cup \{i\}}) = A$, which happens with probability zero.

For all $i \in U, A, B \subseteq  U \setminus i$, the following bounds hold:
\[   \frac{ \mu_{|B| + 1} [\nabla \ell(\x^A)]_i^2 }{\nu_{|A|+ 1} [\nabla \ell(\x^B)]_i^2}  \leq \frac{ G^\ell(i | A) }{G^\ell(i | B) } \leq \frac{\nu_{|B|+ 1}  [\nabla \ell(\x^A)]_i^2}{ \mu_{|A| + 1} [\nabla \ell(\x^B)]_i^2}\]
$G^\ell$ is then $(\alpha_G, \beta_G)$-weakly DR-modular with $\alpha_G := \min_{A \subseteq B \subseteq  U, i \in U \setminus B} \frac{ \mu_{|B| + 1} [\nabla \ell(\x^A)]_i^2 }{\nu_{|A|+ 1} [\nabla \ell(\x^B)]_i^2} > 0$ and $\beta_G := \min_{A \subseteq B \subseteq  U, i \in U \setminus B} \frac{ \mu_{|A| + 1} [\nabla \ell(\x^B)]_i^2 }{\nu_{|B|+ 1} [\nabla \ell(\x^A)]_i^2} > 0$.
\end{proof}

\subsection{Proofs for Section \ref{sec:BO}} \label{sect:AppBOApplicationProofs}

To prove Proposition \ref{prop:VarRed-WDRMod}, we first show that the objective in  noisy column subset selection problems is weakly DR-modular, generalizing the result of \cite{Sviridenko2017}. We then show that the variance reduction function can be written as a noisy column subset selection objective. \\

We start by giving explicit expressions for the marginals of the  objective in  noisy column subset selection problems.\\

\begin{proposition}\label{prop:margLS}
Let $\ell(\x) := \frac{1}{2} \| \y - \A \x\|_2^2 + \frac{\sigma^2}{2} \|\x\|^2$ for some $\sigma \geq 0$ and $G(S) = \ell(0) - \min_{\supp(\x) \subseteq S} \ell(\x)$, then 
\begin{align*}
G(i | S) &= [\x^{S \cup i}(\y)]^2_i ~ \phi(S,i) = \Bigl( \frac{\y^\top R^S(\ab_i)}{2\sqrt{\phi(S,i)}} \Bigl)^2,
\end{align*}
where $\ab_i$ is the ith column of $\A$, $\x^S(\ab_i) :=\argmin_{\supp(\x) \subseteq S} \frac{1}{2} \| \ab_i - \A \x\|_2^2 + \frac{\sigma^2}{2}  \|\x\|^2$ is the vector of optimal regression coefficients,   $R^S(\ab_i)= \ab_i - \A \x^S(\ab_i)$ the corresponding residual, and  $\phi(S,i) =  \frac{1}{2} \| R^S(\ab_i) \|^2 + \frac{\sigma^2}{2}\| \x^S(\ab_i)\|^2 + \frac{\sigma^2}{2}$. 
\end{proposition}
\begin{proof}
Given $i \in V, S \subseteq V \setminus i$, let $\x^S(\y):=\argmin_{\supp(\x) \subseteq S} \ell(\x)$,  and  $R^S(\y) = \y - \A \x^S(\y)$ the corresponding residual. Let $\gamma = [\x^{S \cup i}(\y)]_i$, then we can write 
\begin{align*}
\y &=  \A \x^{S \cup i}(\y) + R^{S \cup i}(\y)  \\
&=  \A_S [\x^{S \cup i}(\y)]_S + \ab_i \gamma + R^{S \cup i}(\y)  \\
&=  \A_S ( [\x^{S \cup i}(\y)]_S + \gamma \x^S(\ab_i))+ \gamma R^S(\ab_i)  + R^{S \cup i}(\y)  \\
\end{align*}
By optimality conditions we have $- \A_{S \cup i}^\top R^{S \cup i}(\y)  + \sigma^2 \x^{S \cup i}(\y) =0$ and $- \A_S^\top R^{S}(\ab_i)  + \sigma^2 \x^{S }(\ab_i) =0$. Let $\hat{\x}^S(\y) =  [\x^{S \cup i}(\y)]_S + \gamma \x^S(\ab_i)$, then  $\hat{\x}^S(\y)$ satisfies the constraint $\supp(\hat{\x}^S(\y)) = S$ and the optimality condition $\A_S^\top (\A_S\hat{\x}^S(\y) - \y ) + \sigma^2 \hat{\x}^S(\y) 
 = 0$. We can see this by plugging in the expression for $\y$ and using the optimality conditions on $\x^{S \cup i}(\y)$ and $\x^{S }(\ab_i)$.
\begin{align*}
\A_S^\top (\A_S\hat{\x}^S(\y) - \y ) + \sigma^2 \hat{\x}^S(\y) &= -\A_S^\top( \gamma R^S(\ab_i)  + R^{S \cup i}(\y))  + \sigma^2( [\x^{S \cup i}(\y)]_S + \gamma \x^S(\ab_i))= 0
\end{align*}
Hence $\hat{\x}^S(\y)= {\x}^S(\y)$. By the optimality condition on $\x^{S \cup i}(\y)$, we also have 
\begin{align*}
R^S(\ab_i)^\top R^{S \cup i}(\y) &= \ab_i^\top R^{S \cup i}(\y) - \x^S(\ab_i)^\top \A_S^\top R^{S \cup i}(\y)\\
&= \sigma^2 \gamma - \sigma^2 \x^S(\ab_i)^\top [\x^{S \cup i}(\y)]_S
\end{align*}
The marginals are thus given by 
\begin{align*}
G(i | S) &= \ell(  {\x}^S(\y)) -  \ell(\x^{S \cup i}(\y)) \\
&= \frac{1}{2} \|\gamma R^S(\ab_i) +  R^{S \cup i}(\y) \|_2^2 + \frac{\sigma^2}{2}  \|[\x^{S \cup i}(\y)]_S + \gamma \x^S(\ab_i)\|^2 -  \frac{1}{2} \| R^{S \cup i}(\y) \|_2^2 -  \frac{\sigma^2}{2}  \|\x^{S \cup i}(\y) \|^2\\
&= \frac{\gamma^2}{2} \| R^S(\ab_i) \|^2  + \sigma^2 \gamma^2 - \sigma^2 \gamma \x^S(\ab_i)^\top [\x^{S \cup i}(\y)]_S + \frac{\sigma^2}{2} \gamma^2 \|  \x^S(\ab_i)\|^2 + \sigma^2 \gamma \x^S(\ab_i)^\top[\x^{S \cup i}(\y)]_S  -  \frac{\sigma^2}{2}  \gamma^2\\
&= \gamma^2 (  \frac{1}{2} \| R^S(\ab_i) \|^2 + \frac{\sigma^2}{2}\|  \x^S(\ab_i)\|^2 + \frac{\sigma^2}{2}  )
\end{align*}
Hence $G(i|S) = [\x^{S \cup i}(\y)]^2_i \phi(S,i)$. 

By the optimality condition on $\x^S(\ab_i)$ we also have:
\begin{align*}
&~~~~~\frac{1}{2} \y^\top R^S(\ab_i) \\
&= \frac{1}{2} (R^{S \cup i}(\y) + \A_{S \cup i} \x^{S \cup i}(\y))^\top R^S(\ab_i)\\
&=  \frac{1}{2} \Bigl( \sigma^2 \gamma - \sigma^2 \x^S(\ab_i)^\top [\x^{S \cup i}(\y)]_S  + [\x^{S \cup i}(\y)]_S^\top \A_S^\top  R^S(\ab_i) + [\x^{S \cup i}(\y)]_i^\top \ab_i^\top  R^S(\ab_i) \Bigl) \\
&=  \frac{1}{2} \Bigl( \sigma^2 \gamma - \sigma^2 \x^S(\ab_i)^\top [\x^{S \cup i}(\y)]_S  + \sigma^2[\x^{S \cup i}(\y)]_S^\top \x^S(\ab_i) + [\x^{S \cup i}(\y)]_i (R^S(\ab_i) + \A_S \x^S(\ab_i))^\top  R^S(\ab_i) \Bigl) \\
&=  \frac{1}{2} \Bigl( \sigma^2 \gamma  + \gamma \| R^S(\ab_i) \|_2^2 + \gamma \sigma^2 \|\x^S(\ab_i)\|_2^2 \Bigl) \\
\end{align*}
Hence $\Bigl( \frac{\y^\top R^S(\ab_i)}{2\sqrt{\phi(S,i)}} \Bigl)^2 = \gamma^2 {\phi(S,i)} = G(i|S)$.
\end{proof}

\begin{proposition}\label{prop:colSel-WDRMod}
Given a positive-definite matrix $\A$, 
let $\ab_i$ be the ith column of $\A$, and $\ell_i(\x) := \frac{1}{2} \| \ab_i - \A \x\|_2^2 + \frac{\sigma^2}{2} \|\x\|^2$ for all $i \in V$, for some $\sigma \geq 0$.
Then the function $G(S) = \sum_{i \in V} \ell(0) - \min_{\supp(\x) \subseteq S} \ell_i(\x)$ is a non-decreasing $(\beta, \beta)$-weakly DR-modular function, with $\beta = (\frac{\lambda^2_{\min}(\A)}{\sigma^2 + \lambda^2_{\min}(\A)})^2 \frac{1}{\kappa^2(\A)}$, where $\kappa(\A) = \lambda_{\max}(\A) / \lambda_{\min}(\A)$ is the condition number of $\A$. 
\end{proposition}
\begin{proof}
For all $j \in V$, let $G_j(S):= \ell(0) - \min_{\supp(\x) \subseteq S} \ell_j(\x)$, then we can write $G(S) = \sum_{j \in V} G_j(S)$.
Given $i \in V, S \subseteq V \setminus i$, let $\x^S(\ab_i) :=\argmin_{\supp(\x) \subseteq S} \ell_i(\x)$ be the optimal regression coefficients,   $R^S(\ab_i)= \ab_i - \A \x^S(\ab_i)$ the corresponding residual. 
By Proposition \ref{prop:margLS}, we have for all $j \in V$:
\begin{align*}
G_j(i | S) &= [\x^{S \cup i}(\ab_j)]^2_i ~ \phi(S,i) = \Bigl( \frac{\ab_j^\top R^S(\ab_i)}{2\sqrt{\phi(S,i)}} \Bigl)^2,
\end{align*}
where $\phi(S,i) =  \frac{1}{2} \| R^S(\ab_i) \|^2 + \frac{\sigma^2}{2}\| \x^S(\ab_i)\|^2 + \frac{\sigma^2}{2}$. Note that $\phi(S,i) >0$ since $\A$ is positive definite (columns are linearly independent).

In the noiseless case $\sigma = 0$, we have $\x^{S \cup i}(\ab_i) = \1_i$. In the noisy case $\sigma > 0$, we have by optimality conditions 
\begin{align*}
(\A_{S \cup i}^\top \A_{S \cup i} + \sigma^2 I) \x^{S \cup i}(\ab_i) &= \A_{S \cup i}^\top \ab_i\\
(\A_{S \cup i}^\top \A_{S \cup i} + \sigma^2 I) \x^{S \cup i}(\ab_i) &= (\A_{S \cup i}^\top \A_{S \cup i}  + \sigma^2 I)\1_i - \sigma^2 \1_i\\
\x^{S \cup i}(\ab_i) &= \1_i - \sigma^2  (\A_{S \cup i}^\top \A_{S \cup i}  + \sigma^2 I)^{-1} \1_i\\
\end{align*}
Since $ (\sigma^2 + \lambda^2_{\min}(\A))^{-1} I \succcurlyeq (\A_{S \cup i}^\top \A_{S \cup i}  + \sigma^2 I)^{-1} \succcurlyeq (\sigma^2 + \lambda^2_{\max}(\A))^{-1} I$, we have 
$$ 1 - \frac{\sigma^2}{\sigma^2 + \lambda^2_{\min}(\A)} \leq [\x^{S \cup i}(\ab_i)]_i \leq 1 - \frac{\sigma^2}{\sigma^2 + \lambda^2_{\max}(\A)}.$$

We will construct two unit vectors $\y, \z$ such that $\tfrac{1}{2}  (\frac{\lambda^2_{\min}(\A)}{\sigma^2 + \lambda^2_{\min}(\A)} )^2 \| \A \y \|^2_2  \leq G(i|S) \leq \frac{1}{2} \|\A \z\|^2_2$.

Let $w_j =  \frac{\ab_j^\top R^S(\ab_i)}{2\sqrt{\phi(S,i)}}, \forall j \in V$ and $\z = \w / \|\w\|_2$. Hence $\| \z \|_2 = 1$ and
\begin{align*}
\tfrac{1}{2} R^S(\ab_i)^\top \A \z &= \tfrac{1}{2} \sum_{j \in V} R^S(\ab_i)^\top \ab_j \frac{w_j}{\|\w\|_2}\\
&= \sqrt{\phi(S,i)} \sum_{j \in V} \frac{w_j^2}{\|\w\|_2}\\
&= \sqrt{\phi(S,i)} {\|\w\|_2}.\\
\end{align*}
Note that $\| \w \|_2^2 = G(i|S)$ and $\phi(S,i) \geq \frac{1}{2} \| R^S(\ab_i) \|^2 $. Then by Cauchy-Schwartz inequality, we have: 
\begin{align*}
 G(i|S) &\leq \frac{ \| R^S(\ab_i) \|^2 \|\A \z \|^2}{4 \phi(S,i)}\\
 &\leq \tfrac{1}{2}  \|\A \z \|^2.
\end{align*}

Let $v_S = \x^S(\ab_i), v_i = -1$ and zero elsewhere, and $y = v / \|v\|_2$. Hence $\| \y \|_2 = 1, \| v\|_2 \geq 1$ and
\begin{align*}
\| \A \y \|_2 &= \frac{\| R^S(\ab_i) \|_2}{\| v \|_2}\\
&\leq \| R^S(\ab_i) \|_2.
\end{align*}
Note that $G(i|S) \geq G_i(i|S) \geq (1 - \frac{\sigma^2}{\sigma^2 + \lambda_{\min}(\A)} )^2 ~ \phi(S,i) \geq \tfrac{1}{2}  (\frac{\lambda_{\min}(\A)}{\sigma^2 + \lambda_{\min}(\A)} )^2 \| R^S(\ab_i) \|^2_2  \geq \tfrac{1}{2}  (\frac{\lambda_{\min}(\A)}{\sigma^2 + \lambda_{\min}(\A)} )^2 \| \A \y \|^2_2 $.

The proposition follows then from 
\begin{align*}
\tfrac{1}{2}  (\frac{\lambda^2_{\min}(\A)}{\sigma^2 + \lambda^2_{\min}(\A)} )^2 \lambda^2_{\min}(\A) = \tfrac{1}{2}  (\frac{\lambda^2_{\min}(\A)}{\sigma^2 + \lambda^2_{\min}(\A)} )^2  \max_{\| \y\|_2=1}\| \A \y \|^2_2  \leq G(i|S) \leq \max_{\| \z\|_2=1}\tfrac{1}{2} \|\A \z\|^2_2 = \tfrac{1}{2}  \lambda^2_{\max}(\A).
\end{align*}

\end{proof}

For the special case of $\sigma = 0$, we recover the result of \cite{Sviridenko2017}.\\

\begin{corollary}
Given a positive-definite kernel matrix $\Kb$, we define for any  $i \in V$, 
$\ell_i(\z) = \| \y - \Kb^{1/2} \z\|_2^2 + \sigma^2 \| \z \|_2^2$ with $\y= \Kb^{1/2} \1_i$, and $ \Kb^{1/2}$ is the principal square root of $\Kb$. Then, we can write the variance reduction function $G(S) = \sum_{i \in V} \sigma^2(\x_i) - \sigma^2_S(\x_i) = \sum_{i \in V}  \ell(\0) - \min_{\supp(\z) \subseteq S} \ell(\z)$. Then  $G$  is a non-decreasing $(\beta, \beta)$-weakly DR-modular function, with $\beta = (\frac{\lambda_{\min}(\Kb)}{\sigma^2 + \lambda_{\min}(\Kb) })^2\frac{\lambda_{\min}(\Kb)}{\lambda_{\max}(\Kb)} $, where  $\lambda_{\max}(\Kb)$ and $\lambda_{\min}(\Kb)$ are the largest and smallest eigenvalues of $\Kb$.
\end{corollary}
\begin{proof}
For a fixed $i \in V, S \subseteq V \setminus i$, let $\z^S :=\argmin_{\supp(x) \subseteq S} \ell_i(\z)$. Then by optimality conditions $\z^S = (\Kb_S + \sigma^2 \I_S)^{-1} \kb_S(\x_i)$. Hence $\ell(\z^S ) = \|\y \|_2^2  - 2 \y^T \Kb_S^{1/2} \z^S  + (\z^S )^\top (\Kb_S + \sigma^2 \I_S) \z^S  = \|\y \|_2^2 - \kb_S(\x_i) (\Kb_S + \sigma^2 \I_S)^{-1} \kb_S(\x_i)$.  It follows then that $\sigma^2(\x_i) - \sigma^2_S(\x_i) = \ell(\0) - \min_{\supp(\z) \subseteq S} \ell(\z)$. 
The corollary then follows from Proposition \ref{prop:colSel-WDRMod}. 
\end{proof}

\end{document}